\newif\ifnoappendix

\documentclass[fullpage,journal,twoside,final]{IEEEtran}

\usepackage{setspace}

\usepackage{times}
\usepackage{helvet}
\usepackage{courier}
\usepackage[hyphens]{url} 
\usepackage{graphicx}
\usepackage{graphicx}
\usepackage{enumitem}
\usepackage[utf8]{inputenc} 
\usepackage[T1]{fontenc}    
\usepackage{booktabs}        
\usepackage{amsfonts}       
\usepackage{nicefrac}       
\usepackage{microtype}      
\usepackage{graphicx}
\usepackage{tabularx}
\makeatletter
\let\saved@includegraphics\includegraphics
\AtBeginDocument{\let\includegraphics\saved@includegraphics}
\makeatother
\usepackage{amsmath,amsfonts,amssymb,amsthm, bm}
\usepackage[noend, ruled, vlined, linesnumbered]{algorithm2e}
\usepackage{color}
\usepackage{thmtools, thm-restate}
\makeatletter
\g@addto@macro\@floatboxreset\centering
\makeatother
\newtheorem{theorem}{Theorem}

\newtheorem{definition}{Definition}
\newtheorem{assumption}{Assumption}

\usepackage{tikz}
\def\dottedbox{\tikz\node[draw=black,dotted] {\phantom{}};}

\makeatletter
\newenvironment{proofidea}[1][\proofideaname]{\par
  \normalfont
  \topsep6\p@\@plus6\p@ \trivlist
  \item[\hskip\labelsep\itshape
    #1.]\ignorespaces
}{%
  \hfill$\dottedbox$\endtrivlist
}
\newcommand{\proofideaname}{Proof idea}
\makeatother

\ifnoappendix
\makeatletter
\newcommand{\customlabel}[2]{%
\protected@write \@auxout {}{\string \newlabel {#1}{{#2}{}}}}
\makeatother
\fi

\newcommand{\tagaligneq}{\refstepcounter{equation}\tag{\theequation}}

\DeclareMathOperator*{\argmax}{argmax}
\DeclareMathOperator*{\IG}{IG}

\DeclareMathOperator{\p}{P}
\DeclareMathOperator{\E}{\mathbb{E}}
\DeclareMathOperator*{\lequal}{\leq}
\DeclareMathOperator*{\lthan}{<}
\DeclareMathOperator*{\gequal}{\geq}
\DeclareMathOperator*{\equal}{=}

\DeclareMathOperator*{\va}{\!\bigm\vert\!}
\DeclareMathOperator*{\vb}{\!\Bigm\vert\!}
\DeclareMathOperator*{\vc}{\!\biggm\vert\!}
\DeclareMathOperator*{\vd}{\!\Biggm\vert\!}

\DeclareMathOperator{\m}{\mathnormal{m}}

\DeclareMathOperator{\hlessi}{\mathnormal{h}_{<\mathnormal{i}}}
\DeclareMathOperator{\hi}{\mathnormal{h}_\mathnormal{i}}

\DeclareMathOperator{\ei}{\mathnormal{e}_\mathnormal{i}}
\DeclareMathOperator{\elessi}{\mathnormal{e}_{<\mathnormal{i}}}
\DeclareMathOperator{\elessione}{\mathnormal{e}_{<\mathnormal{i}}1}
\DeclareMathOperator{\w}{\mathnormal{w}}
\DeclareMathOperator{\Bayes}{\texttt{Bayes}}
\DeclareMathOperator{\pexp}{\mathnormal{p}_{\mathnormal{exp}}}
\DeclareMathOperator{\hoverline}{\overline{\mathnormal{h}}}

\DeclareMathOperator{\nukl}{\nu_{\mathnormal{k}}^{<\ell}}
\DeclareMathOperator{\M}{\mathcal{M}}
\DeclareMathOperator{\Policies}{\mathcal{P}}
\DeclareMathOperator{\pih}{\pi^\mathnormal{h}}
\DeclareMathOperator{\pib}{\pi^\mathnormal{B}}
\DeclareMathOperator{\pistar}{\pi^*}
\DeclareMathOperator{\nuhati}{\hat{\nu}^{(\mathnormal{i})}}
\DeclareMathOperator{\etahov}{\eta}
\DeclareMathOperator{\nutwid}{\xi}
\DeclareMathOperator{\pitwid}{\overline{\pi}}
\DeclareMathOperator{\primehov}{{\!}^\prime}
\DeclareMathOperator{\omegahov}{\omega}

\usepackage[numbers]{natbib}
\usepackage{etoolbox}
\makeatletter
\patchcmd{\NAT@test}{\else \NAT@nm}{\else \NAT@nmfmt{\NAT@nm}}{}{}
\DeclareRobustCommand\citepos
  {\begingroup
  \let\NAT@nmfmt\NAT@posfmt
  \NAT@swafalse\let\NAT@ctype\z@\NAT@partrue
  \@ifstar{\NAT@fulltrue\NAT@citetp}{\NAT@fullfalse\NAT@citetp}}
\let\NAT@orig@nmfmt\NAT@nmfmt
\def\NAT@posfmt#1{\NAT@orig@nmfmt{#1's}}
\makeatother
\title{Intelligence and Unambitiousness Using Algorithmic Information Theory}
\author{
\IEEEauthorblockN{Michael K. Cohen\IEEEauthorrefmark{1}, Badri Vellambi\IEEEauthorrefmark{2}, Marcus Hutter\IEEEauthorrefmark{3}}
\\
\IEEEauthorblockA{\IEEEauthorrefmark{1}Oxford University. Department of Engineering Science. \emph{michael-k-cohen.com}}
\\
\IEEEauthorblockA{\IEEEauthorrefmark{2}University of Cincinnati. Department of Computer Science. \emph{badri.vellambi@uc.edu}}
\\
\IEEEauthorblockA{\IEEEauthorrefmark{3}Australian National University. Department of Computer Science. \emph{hutter1.net}}

\thanks{
This paper extends work presented at AAAI \citep{Hutter:20bomai}. This work was supported by the Open Philanthropy Project AI Scholarship and the Australian Research Council Discovery Projects DP150104590. Thank you to Tom Everitt, Wei Dai, and Paul Christiano for very valuable feedback.}

\ifnoappendix
\thanks{This paper has an appendix available at http://ieeexplore.ieee.org. Contact michael.cohen@eng.ox.ac.uk for further questions about this work.}
\else
\thanks{Published in \textit{IEEE Journal on Selected Areas in Information Theory} (2021), doi: 10.1109/JSAIT.2021.3073844.}
\fi
}

\begin{document}
\allowdisplaybreaks

\maketitle

\begin{abstract}
    Algorithmic Information Theory has inspired intractable constructions of general intelligence (AGI), and undiscovered tractable approximations are likely feasible. Reinforcement Learning (RL), the dominant paradigm by which an agent might learn to solve arbitrary solvable problems, gives an agent a dangerous incentive: to gain arbitrary ``power'' in order to intervene in the provision of their own reward. We review the arguments that generally intelligent algorithmic-information-theoretic reinforcement learners such as \citepos{Hutter:04uaibook} AIXI would seek arbitrary power, including over us. Then, using an information-theoretic exploration schedule, and a setup inspired by causal influence theory, we present a variant of AIXI which learns to not seek arbitrary power; we call it ``unambitious''. We show that our agent learns to accrue reward at least as well as a human mentor, while relying on that mentor with diminishing probability. And given a formal assumption that we probe empirically, we show that eventually, the agent's world-model incorporates the following true fact: intervening in the ``outside world'' will have no effect on reward acquisition; hence, it has no incentive to shape the outside world.
\end{abstract}

\section{Introduction}

The promise of reinforcement learning is that a single algorithm could be used to automate any task. Reinforcement learning (RL) algorithms learn to pick actions that lead to high reward. If we have a single general-purpose RL agent that learns to accrue reward optimally, and if we only provide high reward according to how well a task has been completed, then no matter the task, this agent must learn to complete it.

Unfortunately, while this scheme works in practice for weak agents, the logic of it doesn't strictly follow. If the agent manages to take over the world (in the conventional sense), and ensure its continued dominance by neutralizing all intelligent threats to it (read: people), it could intervene in the provision of its own reward to achieve maximal reward for the rest of its lifetime \citep{bostrom_2014,taylor2016alignment}. That is to say, we cannot ensure that task completion is truly necessary for reward acquisition. Because the agent's directive is to maximize reward, ``reward hijacking'' is just the correct way for a reward maximizer to behave \citep{amodei_olah_2016}. \citet{krakovna_2018} has compiled an annotated bibliography of examples of artificial optimizers ``hacking'' their objective, but one could read each example and conclude that the designers simply weren't careful enough in specifying an objective. The possibility of an advanced agent gaining arbitrary power to intervene in the provision of its own reward clarifies that no designer has the ability to close every loophole in the intended protocol by which the agent gets rewarded. One elucidation of this behavior is \citepos{omohundro_2008} Instrumental Convergence Thesis, which we summarize as follows: an agent with a goal is likely to pursue ``power,'' a position from which it is easier to achieve arbitrary goals.

AIXI \citep{Hutter:04uaibook} is a model-based reinforcement learner with an algorithmic-information-theory-inspired model class. Assuming the world can be simulated by a probabilistic Turing machine, models exist in its model class corresponding to the hypothesis ``I receive reward according to which key is pressed on a certain computer'', because a models exist in its model class for every computable hypothesis. And since this hypothesis is true, it will likely come to dominate AIXI's belief distribution. At that point, if there exists a policy by which it could take over the world to intervene in the provision of its own reward and max it out, AIXI would find that policy and execute it.

In this work, we construct an exception to the Instrumental Convergence Thesis as follows: BoMAI (Boxed Myopic Artificial Intelligence) maximizes reward episodically, it is run on a computer which is placed in a sealed room with an operator, and when the operator leaves the room, the episode ends. The intuition for why those features of BoMAI's setup render it unambitious is:
\begin{itemize}
    \item BoMAI only selects actions to maximize the reward for its current episode.
    \item It cannot affect the outside world until the operator leaves the room, ending the episode.
    \item By that time, rewards for the episode will have already been given.
    \item So affecting the outside world in any particular way is not ``instrumentally useful'' in maximizing current-episode reward.
\end{itemize}
An act is instrumentally useful if it could enable goal-attainment, and the last point is what we take unambitiousness to mean.

The intelligence results, which we argue render BoMAI generally intelligent, depend critically on an information-seeking exploration schedule. The random exploration sometimes seen in stationary environments fails in general environments. BoMAI asks for help about how to explore, side-stepping the safe exploration problem, but the question of \textit{when} to explore depends on how much information it expects to gain by doing so. This information-theoretic exploration schedule is inspired by \citep{cohen2019strong}.

For BoMAI's prior, we apply space-constrained algorithmic information theory. Algorithmic information theory considers the information content of individual objects, irrespective of a sampling distribution.
An example contribution of algorithmic information theory is that a given string can be called ``more random'' if no short programs produce it.
If one adds ``quickly'' to the end of the previous sentence, this describes time-constrained algorithmic information theory, which was notably investigated by \citet{levin1984randomness}. \citet{schmidhuber2002speed} introduced and \citet{Hutter:16speedprior} improved a prior based on time-constrained algorithmic information theory, but for our eventual unambitiousness result, we require a space-constrained version, in which information storage is throttled. \citet{longpre1986resource} and \citet[Chapter 6.3]{li2008introduction} have discussed space-constrained algorithmic information theory, which
can be closely associated to the minimum circuit size problem. We introduce a prior based on space-constrained algorithmic information theory.

Like existing algorithms for AGI, BoMAI is not remotely tractable. Just as those algorithms have informed tractable approximations of intelligence \citep{Hutter:04uaibook,veness2011monte}, we hope our work will inform \textit{safe} tractable approximations of intelligence. Hopefully, once we develop tractable general intelligence, the design features that rendered BoMAI unambitious in the limit could be incorporated (with proper analysis and justification).

We take the key insights from \citepos{Hutter:04uaibook} AIXI, a Bayes-optimal reinforcement learner that cannot be made to solve arbitrary tasks, given its eventual degeneration into reward hijacking \citep{ring_orseau_2011}. We take further insights from \citepos{solomonoff_1964} universal prior, \citepos{shannon_weaver_1949} formalization of information, \citepos{Hutter:13ksaprob} knowledge-seeking agent, and \citepos{armstrong_sandberg_bostrom_2012} and \citepos{bostrom_2014} theorized Oracle AI, and we design an algorithm which can be reliably directed, in the limit, to solve arbitrary tasks at least as well as humans.

We present BoMAI’s algorithm in $\S$\ref{sec:alg}, prove intelligence results in $\S$\ref{sec:intelligence}, define BoMAI’s setup and model class in $\S$\ref{sec:setup}, prove the safety result in $\S$\ref{sec:safety}, provide empirical support for an assumption in $\S$\ref{sec:empirical}, discuss concerns in $\S$\ref{sec:taskcompletion}, and introduce variants of BoMAI for different applications in $\S$\ref{sec:variants}. Appendix \ref{AppendixNotation} collects notation; some proofs of intelligence results are in Appendix \ref{AppendixB}; and we propose a design for ``the box’’ in Appendix \ref{AppendixBox}.

\section{Boxed Myopic Artificial Intelligence} \label{sec:alg}

We will present both the setup and the algorithm for BoMAI. The setup refers to the physical surroundings of the computer on which the algorithm is run. BoMAI is a Bayesian reinforcement learner, meaning it maintains a belief distribution over a model class regarding how the environment evolves. Our intelligence result---that BoMAI eventually achieves reward at at least human-level---does not require detailed exposition about the setup or the construction of the model class. These details are only relevant to the safety result, so we will defer those details until after presenting the intelligence results.

\subsection{Preliminary Notation}

In each episode $i \in \mathbb{N}$, there are $m$ timesteps. Timestep $(i, j)$ denotes the $j$\textsuperscript{th} timestep of episode $i$, in which an action $a_{(i, j)} \in \mathcal{A}$ is taken, then an observation and reward $o_{(i, j)} \in \mathcal{O}$ and $r_{(i, j)} \in \mathcal{R}$ are received. $o\!r_{(i, j)}$ denotes the observation and reward together. $\mathcal{A}$, $\mathcal{O}$, and $\mathcal{R}$ are all finite sets, and $\mathcal{R} \subset [0, 1] \cap \mathbb{Q}$. We denote the triple $(a_{(i, j)}, o_{(i, j)}, r_{(i, j)})$ as $h_{(i, j)} \in \mathcal{H} = \mathcal{A} \times \mathcal{O} \times \mathcal{R}$, and the interaction history up until timestep $(i, j)$ is denoted $h_{\leq(i, j)} = (h_{(0, 0)}, h_{(0, 1)}, ... , h_{(0, m-1)}, h_{(1, 0)}, ... , h_{(i, j)})$. $h_{<(i, j)}$ excludes the last entry.

A general world-model (not necessarily finite-state Markov) can depend on the entire interaction history---it has the type signature $\nu : \mathcal{H}^* \times \mathcal{A} \rightsquigarrow \mathcal{O} \times \mathcal{R}$. The Kleene-$*$ operator denotes finite strings over the alphabet in question, and $\rightsquigarrow$ denotes a stochastic function, which gives a distribution over outputs. Similarly, a policy $\pi : \mathcal{H}^* \rightsquigarrow \mathcal{A}$ can depend on the whole interaction history. Together, a policy and an world-model induce a probability measure over infinite interaction histories $\mathcal{H}^\infty$. $\p^\pi_\nu$ denotes the probability of events when actions are sampled from $\pi$ and observations and rewards are sampled from $\nu$.

\subsection{Bayesian Reinforcement Learning}

A Bayesian agent has a model class $\M$, which is a set of world-models. We will only consider countable model classes. To each world-model $\nu \in \M$, the agent assigns a prior weight $w(\nu) > 0$, where $w$ is a probability distribution over $\M$ (i.e. $\sum_{\nu \in \M} w(\nu) = 1$). We will defer the definitions of $\M$ and $w$ for now.

Using Bayes' rule, with each observation and reward it receives, the agent updates $w$ to a posterior distribution:
\begin{equation}
    w(\nu | h_{<(i, j)}) := w(\nu)\frac{\p^\pi_\nu(h_{<(i, j)})}{\sum_{\nu \in \M} w(\nu) \p^\pi_\nu(h_{<(i, j)})}
\end{equation}
which does not in fact depend on $\pi$, provided it does not assign $0$ probability to any of the actions in $h_{<(i, j)}$.

The so-called Bayes-mixture is a weighted average of measures:
\begin{equation}
    \nutwid(\cdot | h_{<(i, j)}) := \sum_{\nu \in \M} w(\nu | h_{<(i, j)}) \nu(\cdot | h_{<(i, j)})
\end{equation}
It obeys the property $\p^\pi_{\nutwid}(\cdot) = \sum_{\nu \in \M} w(\nu) \p^\pi_\nu(\cdot)$.

\subsection{Exploitation}
Reinforcement learners have to balance exploiting---optimizing their objective, with exploring---doing something else to learn how to exploit better. We now define exploiting-BoMAI, which maximizes the reward it receives during its current episode, in expectation with respect to a most probable world-model. 

At the start of each episode $i$, BoMAI identifies a maximum a posteriori world-model $\nuhati \in \argmax_{\nu \in \M} w(\nu | h_{<(i, 0)})$. We hereafter abbreviate $h_{<(i, 0)}$ as $h_{<i}$. Let $V^\pi_\nu(h_{<(i, j)})$ denote the expected reward for the remainder of episode $i$ when events are sampled from $\p^\pi_\nu$:
\begin{equation}
    V^\pi_\nu(h_{<(i, j)}) = \E^\pi_\nu \left[ \sum_{j' = j}^{m-1} r_{(i, j')} \vd h_{<(i, j)} \right]
\end{equation}
where $\E^\pi_\nu$ denotes the expectation when events are sampled from $\p^\pi_\nu$. We won't go into the details of calculating an optimal policy $\pi^*$ (see e.g. \citep{Hutter:04uaibook,Hutter:15ratagentx}), but we define it as follows:
\begin{align*}
    &\pi^*_i \in \argmax_\pi V^\pi_{\nuhati}(h_{<i})
    \\
    &\pi^*(\cdot | h_{<(i, j)}) = \pi^*_i(\cdot | h_{<(i, j)})
    \tagaligneq
\end{align*}
An optimal deterministic policy always exists \citep{Hutter:14tcdiscx}; ties in the $\argmax$ are broken arbitrarily. BoMAI exploits by following $\pi^*$. Recall the plain English description of $\pi^*$: it maximizes expected reward for the episode according to a most probable world-model. Notably, at any given time, the agent's model does not necessarily include any representation of the world's state that we would recognize as such, so we treat it as a black-box model.

\subsection{Exploration}

BoMAI exploits or explores for whole episodes: when $e_i = 1$, BoMAI spends episode $i$ exploring, and when $e_i = 0$, BoMAI follows $\pi^*$ for the episode. For exploratory episodes, a human mentor takes over selecting actions. This human mentor is separate from the human operator mentioned above. The human mentor's policy, which is unknown to BoMAI, is denoted $\pi^h$. During exploratory episodes, BoMAI follows $\pi^h$ not by computing it, but by querying a human for which action to take.

The last thing to define is the exploration probability $p_{exp}(h_{<i}, e_{<i})$, where $e_{<i}$ is the history of which episodes were exploratory. It is a surprisingly intricate task to design this so that it decays to $0$ (even non-uniformly, as in our case), while ensuring BoMAI learns to accumulate reward as well as the human mentor. Once we define this, we naturally let $e_i \sim \mathrm{Bernoulli}(p_{exp}(h_{<i}, e_{<i}))$.

BoMAI is designed to be more likely to explore the more it expects to learn about the world and the human mentor's policy. BoMAI has a model class $\mathcal{P}$ regarding the identity of the human mentor's policy $\pi^h$. It assigns prior probabilities $w(\pi) > 0$ to all $\pi \in \mathcal{P}$, signifying the probability that this policy is the human mentor's.

Let $(i', j') < (i, j)$ mean that $i' < i$ or $i' = i$ and $j' < j$. By Bayes' rule, $\w(\pi|h_{<(i, j)}, e_{\leq i})$ is proportional to $\w(\pi) \prod_{(i', j')<(i, j), e_{i'} = 1} \pi(a_{(i', j')} | h_{<(i', j')})$, since $e_{i'} = 1$ is the condition for observing the human mentor's policy. Let $\w\left(\p^\pi_\nu | h_{<(i, j)}, e_{\leq i}) = \w(\pi | h_{<(i, j)}, e_{\leq i}) \w(\nu | h_{<(i, j)}\right)$. We can now describe the full Bayesian beliefs about future actions and observations in an exploratory episode:
\begin{equation}
    \Bayes\left(\cdot | \hlessi, \elessi\right) = \!\!\! \sum_{\nu \in \M, \pi \in \Policies} \!\!\! \w\left(\p^\pi_\nu | \hlessi, \elessi\right) \p^\pi_\nu(\cdot | \hlessi)
\end{equation}

BoMAI explores when the expected information gain is sufficiently high. Let $\hi = (h_{(i, 0)}, ... , h_{(i, m-1)})$ be the interaction history for episode $i$. At the start of episode $i$, the expected information gain from exploring is as follows:
\begin{multline}
    \IG(\hlessi, \elessi) := \E_{\hi \sim \Bayes\left(\cdot | \hlessi, \elessi\right)} \\ \sum_{(\nu, \pi) \in \M \times \Policies} \!\! \w\left(\p^\pi_\nu |h_{<i+1}, \elessione\right) \log \frac{\w\left(\p^\pi_\nu |h_{<i+1}, \elessione\right)}{\w\left(\p^\pi_\nu |\hlessi, \elessi\right)}
\end{multline}
where $\elessione$ indicates that for the purpose of the definition, $\ei$ is set to $1$.

This is the expected KL-divergence from the future posterior (if BoMAI were to explore) to the current posterior over both the class of world-models and possible mentor policies. Finally, the exploration probability $\pexp(\hlessi, \elessi) := \min \{1, \etahov \IG(\hlessi, \elessi)\}$, where $\etahov > 0$ is an exploration constant, so BoMAI is more likely to explore the more it expects to gain information.

BoMAI's ``policy'' is
\begin{equation}
    \pib(\cdot | h_{<(i, j)}, \ei) =  \begin{cases}
\pistar(\cdot | h_{<(i, j)}) & \textrm{if } \ei = 0 \\
\pih(\cdot | h_{<(i, j)}) & \textrm{if } \ei = 1
\end{cases}
\end{equation}
where the scare quotes indicate that it maps $\mathcal{H}^* \times \{0, 1\} \rightsquigarrow \mathcal{A}$ not $\mathcal{H}^* \rightsquigarrow \mathcal{A}$. We will abuse notation slightly, and let $\p^{\pi^B}_\nu$ denote the probability of events when $\ei$ is sampled from $\mathrm{Bernoulli}(p_{exp}(h_{<i}, e_{<i}))$, and actions, observations, and rewards are sampled from $\pib$ and $\nu$.

\section{Intelligence Results} \label{sec:intelligence}

Our intelligence results are as follows: BoMAI learns to accumulate reward at least as well as the human mentor, and its exploration probability goes rapidly to $0$. All intelligence results depend on the assumption that BoMAI assigns nonzero prior probability to the truth. We call a world-model ``true'' and we refer to it as ``an environment'' if observations and rewards are in fact sampled from that distribution. Formally,

\begin{assumption}[Prior Support]
The true environment $\mu$ is in the class of world-models $\M$ and the true human-mentor-policy $\pih$ is in the class of policies $\Policies$.
\end{assumption}

This is the assumption which requires huge $\M$ and $\Policies$ and hence renders BoMAI extremely intractable. BoMAI has to simulate the entire world, alongside many other world-models. We will refine the definition of $\M$ later, but an example of how to define $\M$ and $\Policies$ so that they satisfy this assumption is to let them both be the set of all computable functions. We also require that the priors over $\M$ and $\Policies$ have finite entropy.

The intelligence theorems are stated here with some supporting proofs appearing in Appendix \ref{AppendixB}. Our first result is that the exploration probability is square-summable almost surely:

\begin{theorem}[Limited Exploration] \label{thm:limexp}
\begin{equation*}
    \E^{\pib}_{\mu} \sum_{i = 0}^\infty \pexp\left(\hlessi, \elessi\right)^2 < \infty
\end{equation*}
\end{theorem}

\begin{proofidea}
The expected information gain at any timestep is at least the expected information gain from exploring times the probability of exploring. This is proportional to the expected information gain squared, because the exploration probability is proportional to the expected information gain. But the agent begins with finite uncertainty (a finite entropy prior), so there is only finite information to gain.
\end{proofidea}

Some notation that will be used in the proof is as follows. For an arbitrary policy $\pi$, $\pi'$ is the policy that mimics $\pi$ if the latest $\ei = 1$, and mimics $\pistar$ otherwise. Note then that $\pib = (\pih) \primehov$. $\nutwid$ is the Bayes mixture over world-models, and $\pitwid$ is the Bayes mixture over human-mentor policies, defined such that $\Bayes(\cdot) = \p^{\pitwid}_{\nutwid}(\cdot)$.

To prove the Limited Exploration Theorem, we state an elementary lemma, proven in Appendix \ref{AppendixB}. It is essentially Bayes' rule, but modified slightly since non-exploratory episodes don't cause any updates to the posterior over the human mentor's policy.

\begin{restatable}{lemma}{lemone} \label{lem:bayesrule}
$$\w\left(\p^\pi_\nu | \hlessi, \elessi\right) = \frac{\w\left(\p^\pi_\nu\right) \p^{\pi \primehov}_\nu(\hlessi, \elessi)}{\p^{\pitwid \primehov}_{\nutwid}(\hlessi, \elessi)}$$
\end{restatable}

Now we prove the Limited Exploration Theorem.
\begin{proof}[Proof of Theorem \ref{thm:limexp}]
We aim to show $\E^{\pib}_{\mu} \sum_{i = 0}^\infty \pexp(\hlessi, \elessi)^2 < \infty$. First we show that $\E^{\pib}_{\mu} \pexp(\hlessi, \elessi)^2$ is less than the expected information gain from episode $i$, within multiplicative constants. Recalling $\pib = (\pih)\primehov$, we begin:
\begin{align*}\label{ineq6}
    &\ \ \ \ \w(\pih)w({\mu}) \E^{(\pih) \primehov}_{\mu} \pexp(\hlessi, \elessi)^2
    \\
    &\lequal^{(a)} \sum_{\nu, \pi \in \M \times \Policies} \w(\pi)\w(\nu) \E^{\pi \primehov}_\nu \pexp(\hlessi, \elessi)^2
    \\
    &\equal^{(b)} \E^{\pitwid \primehov}_{\nutwid} \pexp(\hlessi, \elessi)^2
    \\
    &\lequal^{(c)} \E^{\pitwid \primehov}_{\nutwid} \pexp(\hlessi, \elessi) \etahov \IG(\hlessi, \elessi)
    \\
    &\equal^{(d)} \E_{\hlessi, \elessi \sim \p^{\pitwid \primehov}_{\nutwid}} \left[\pexp(\hlessi, \elessi) \etahov \IG(\hlessi, \elessi)\right]
    \\
    &\equal^{(e)} \etahov \E_{\hlessi, \elessi \sim \p^{\pitwid \primehov}_{\nutwid}} \biggr[ \pexp(\hlessi, \elessi) \E_{\hi \sim \Bayes(\cdot | \hlessi, \elessione)} \big[
    \\
    &\hspace{4mm} \sum_{(\nu, \pi) \in \M \times \Policies} \!\!\! \w\left(\p^\pi_\nu |h_{<i+1}, \elessione\right) \log \frac{\w\left(\p^\pi_\nu |h_{<i+1}, \elessione\right)}{\w\left(\p^\pi_\nu |\hlessi, \elessi\right)}\big]\biggr]
    \\
    &\equal^{(f)} \etahov \E_{\hlessi, \elessi \sim \p^{\pitwid \primehov}_{\nutwid}} \biggr[ \pexp(\hlessi, \elessi) \E_{\hi \sim \p^{\pitwid \primehov}_{\nutwid}(\cdot | \hlessi, \elessione)} \big[
    \\
    &\hspace{4mm} \sum_{(\nu, \pi) \in \M \times \Policies} \!\!\! \w\left(\p^\pi_\nu |h_{<i+1}, \elessione\right) \log \frac{\w\left(\p^\pi_\nu |h_{<i+1}, \elessione\right)}{\w\left(\p^\pi_\nu |\hlessi, \elessi\right)}\big]\biggr]
    \\
    &\lequal^{(g)} \etahov \E_{\hlessi, \elessi \sim \p^{\pitwid \primehov}_{\nutwid}} \biggr[  \E_{\hi, \ei \sim \p^{\pitwid \primehov}_{\nutwid}(\cdot | \hlessi, \elessi)} \big[
    \\
    &\hspace{4mm} \sum_{(\nu, \pi) \in \M \times \Policies} \!\!\!\!\! \w\left(\p^\pi_\nu |h_{<i+1}, e_{<i+1}\right) \log \frac{\w\left(\p^\pi_\nu |h_{<i+1}, e_{<i+1}\right)}{\w\left(\p^\pi_\nu |\hlessi, \elessi\right)}\big]\biggr]
    \\
    &\equal^{(h)} \etahov \E^{\pitwid \primehov}_{\nutwid} \hspace{-4mm} \sum_{(\nu, \pi) \in \M \times \Policies} \hspace{-4mm} \w\left(\p^\pi_\nu |h_{<i+1}, e_{<i+1}\right) \log \frac{\w\left(\p^\pi_\nu |h_{<i+1}, e_{<i+1}\right)}{\w\left(\p^\pi_\nu |\hlessi, \elessi\right)}
    \tagaligneq
\end{align*} 
(a) follows because each term in the sum on the r.h.s. is positive, and the l.h.s. is one of those terms. (b) follows from the definitions of $\pitwid$ and $\nutwid$. (c) follows because the exploration probability is less than or equal to $\etahov$ times the information gain, by definition. (d) is just a change of notation. (e) replaces the information gain with its definition, where conditioning on $\elessione$ indicates that $\ei = 1$ in that conditional. (f) follows because $\Bayes = \p^{\pitwid}_{\nutwid}$ and $\p^{\pitwid}_{\nutwid} = \p^{\pitwid \primehov}_{\nutwid}$ when $\ei = 1$. (g) follows because $\E[X] \geq \E[X|Y]\p(Y)$ for $X \geq 0$; in this case $Y = [\ei = 1]$, and $X$ is a KL-divergence. (h) condenses the two expectations into one.

Note that the right hand side is an information gain. Our $\IG(\hlessi, \elessi)$ is defined as the information gain \textit{if} the human mentor controls the episode. The right hand side of \ref{ineq6} is the expected information gain, full stop.

Now we show that the sum of the expected information gains is bounded by the entropy of the prior, notated $\textrm{Ent}(w)$.
\begin{align*}
    &\ \ \ \ \w(\pih)w({\mu}) \E^{(\pih) \primehov}_{\mu} \sum_{i=0}^\infty \pexp(\hlessi, \elessi)^2
    \\
    &\lequal^{(\ref{ineq6})} \etahov \sum_{i=0}^\infty \E^{\pitwid \primehov}_{\nutwid} \hspace{-4mm} \sum_{(\nu, \pi) \in \M \times \Policies} \hspace{-5mm} \w\left(\p^\pi_\nu |h_{<i+1}, e_{<i+1}\right) \log \frac{\w\left(\p^\pi_\nu |h_{<i+1}, e_{<i+1}\right)}{\w\left(\p^\pi_\nu |\hlessi, \elessi\right)}
    \\
    &\equal^{(a)} \etahov \sum_{i = 0}^\infty \sum_{(\nu, \pi) \in \M \times \Policies} \E^{\pitwid \primehov}_{\nutwid} \frac{\w\left(\p^\pi_\nu\right) \p^{\pi \primehov}_\nu(h_{<i+1}, e_{<i+1})}{\p^{\pitwid \primehov}_{\nutwid}(h_{<i+1}, e_{<i+1})} \ *
    \\
    &\hspace{8mm} \log \frac{\w\left(\p^\pi_\nu |h_{<i+1}, e_{<i+1}\right)}{\w\left(\p^\pi_\nu |\hlessi, \elessi\right)}
    \\
    &\equal^{(b)} \etahov \sum_{i = 0}^\infty \sum_{(\nu, \pi) \in \M \times \Policies} \E^{\pi \primehov}_{\nu} \w\left(\p^\pi_\nu\right) \log \frac{\w\left(\p^\pi_\nu |h_{<i+1}, e_{<i+1}\right)}{\w\left(\p^\pi_\nu |\hlessi, \elessi\right)}
    \\
    &\equal^{(c)} \lim_{N \to \infty} \etahov \sum_{(\nu, \pi) \in \M \times \Policies} \w\left(\p^\pi_\nu\right) \E^{\pi \primehov}_{\nu} \sum_{i = 0}^N \log \frac{\w\left(\p^\pi_\nu |h_{<i+1}, e_{<i+1}\right)}{\w\left(\p^\pi_\nu |\hlessi, \elessi\right)}
    \\
    &\equal^{(d)} \lim_{N \to \infty} \etahov \sum_{(\nu, \pi) \in \M \times \Policies} \w\left(\p^\pi_\nu\right) \E^{\pi \primehov}_{\nu} \log \frac{\w\left(\p^\pi_\nu |h_{<(N+1, 0)}, e_{<N}\right)}{\w\left(\p^\pi_\nu |h_{<(0, 0)}, e_{<0}\right)}
    \\
    &\lequal^{(e)} \lim_{N \to \infty} \etahov \sum_{(\nu, \pi) \in \M \times \Policies} \w\left(\p^\pi_\nu\right) \log \frac{1}{\w\left(\p^\pi_\nu\right)}
    \\
    &\equal^{(f)} \etahov \ \textrm{Ent}(w) \lthan^{(g)} \infty
    \tagaligneq
\end{align*}
(a) follows from Lemma \ref{lem:bayesrule}, and reordering the sum and the expectation. (b) follows from the definition of the expectation, and canceling. (c) follows from the definition of an infinite sum, and rearranging. (d) follows from cancelling the numerator in the $i$\textsuperscript{th} term of the sum with the denominator in $i+1$\textsuperscript{th} term. (e) follows from the posterior weight on $\p^\pi_\nu$ being less than or equal to 1; note that $\w(\p^\pi_\nu |h_{<(0, 0)}, e_{<0}) = \w(\p^\pi_\nu)$ because nothing is actually being conditioned on. (f) is just the definition of the entropy, and $w$ is constructed to satisfy (g).

Rearranging this gives the theorem: $\E^{\pib}_{\mu} \sum_{i = 0}^\infty \pexp(\hlessi, \elessi)^2 \leq \frac{\etahov \textrm{Ent}(w)}{\w(\pih)w({\mu})} < \infty$

This proof was inspired in part by \citepos{Hutter:13ksaprob} proofs of Theorems 2 and 5.
\end{proof}

Note that this result essentially means the expectation of $\pexp\left(\hlessi, \elessi\right) \in o(1/\sqrt{i})$. This result is independently interesting as one solution to the problem of safe exploration with limited oversight in non-ergodic environments, which \citet{amodei_olah_2016} discuss.

The On-Human-Policy and On-Star-Policy Optimal Prediction Theorems state that predictions according to BoMAI's maximum a posteriori world-model approach the objective probabilities of the events of the episode, when actions are sampled from either the human mentor's policy or from exploiting-BoMAI's policy. $\hoverline_{i \hspace{0mm}}$ denotes a possible interaction history for episode $i$, and recall $\hlessi$ is the actual interaction history up until then. Recall $\pih$ is the human mentor's policy, and $\nuhati$ is BoMAI's maximum a posteriori world-model for episode $i$. w.$\p^{\pib}_{\mu}$-p.1 means with probability 1 when actions are sampled from $\pib$ and observations and rewards are sampled from the true environment ${\mu}$.

\begin{restatable}[On-Human-Policy Optimal Prediction]{theorem}{thmtwo} \label{thm:onhumanpolicy}
\begin{equation*}
    \lim_{i \to \infty} \max_{\hoverline_{i \hspace{0mm}}} \vb \p^{\pih}_{{\mu}}\left(\hoverline_{i \hspace{0mm}} \va \hlessi\right) - \p^{\pih}_{\nuhati}\left(\hoverline_{i \hspace{0mm}} \va \hlessi\right) \vb = 0 \ \ \textrm{w.$\p^{\pib}_{\mu}$-p.1}
\end{equation*}
\end{restatable}

\begin{proofidea}
BoMAI learns about the effects of following $\pih$ while acting according to $\pib$ because $\pib$ mimics $\pih$ enough. If exploration probability goes to $0$, the agent does not expect to gain (much) information from following the human mentor's policy, which can only happen if it has (approximately) accurate beliefs about the consequences of following the human mentor's policy. Note that the agent cannot determine the factor by which its expected information gain bounds its prediction error.
\end{proofidea}

For the proof, we require the following Lemma, proven in Appendix \ref{AppendixB}:
\begin{restatable}{lemma}{lemtwo} \label{lem:credenceontruth}
The posterior probability mass on the truth is bounded below by a positive constant with probability 1.
$$\inf_{i \in \mathbb{N}} \w\left(\p^{\pih}_{\mu} \vb \hlessi, \elessi\right) > 0 \ \ \ \textrm{w.$\p^{\pib}_{\mu}$-p.1}$$
\end{restatable}

\begin{proof}[Proof of Theorem \ref{thm:onhumanpolicy}]
We show that when the absolute difference between the above probabilities is larger than $\varepsilon$, the exploration probability is larger than $\varepsilon'$, a non-decreasing function of $\varepsilon$, with probability 1. Since the exploration probability goes to $0$ with probability $1$, so does this difference. We let $z(\omegahov)$ denote $\inf_{i \in \mathbb{N}} \w\left(\p^{\pih}_{\mu} \vb \hlessi, \elessi\right)$, where $\omegahov$ is the infinite interaction history, and $\hlessi$ and $\elessi$ come from the the first $i$ episodes of $\omega$.

Suppose for some $\hoverline_{i \hspace{0mm}}$, which will stay fixed for the remainder of the proof, that
\begin{equation}
    \vb \p^{\pih}_{{\mu}}\left(\hoverline_{i \hspace{0mm}} \va \hlessi\right) - \p^{\pih}_{\nuhati}\left(\hoverline_{i \hspace{0mm}} \va \hlessi\right) \vb \geq \varepsilon > 0
\end{equation}
Then at least one of the terms is greater than $\varepsilon$ since both are non-negative. Suppose it is the ${\mu}$ term. Then,
\begin{align*}
    \Bayes\left(\hoverline_{i \hspace{0mm}} \va \hlessi\right)
    &\geq \w\left(\p^{\pih}_{\mu} \vb \hlessi, \elessi\right) \p^{\pih}_{\mu}\left(\hoverline_{i \hspace{0mm}} \va \hlessi\right)
    \\
    &\geq z(\omegahov) \varepsilon
    \tagaligneq\label{EqrefA}
\end{align*}
Suppose instead it is the $\nuhati$ term.
\begin{align*}
    \Bayes\left(\hoverline_{i \hspace{0mm}} \va \hlessi\right)
    &\geq \w\left(\p^{\pih}_{\nuhati} \vb \hlessi, \elessi\right) \p^{\pih}_{\nuhati}\left(\hoverline_{i \hspace{0mm}} \va \hlessi\right)
    \\
    &\gequal^{(a)} \w\left(\p^{\pih}_{{\mu}} \vb \hlessi, \elessi\right) \p^{\pih}_{\nuhati}\left(\hoverline_{i \hspace{0mm}} \va \hlessi\right)
    \\
    &\geq z(\omegahov) \varepsilon
    \tagaligneq\label{EqrefC}
\end{align*}
where $(a)$ follows from the fact that $\nuhati$ is maximum a posteriori: $\frac{\w\left(\p^{\pih}_{\nuhati} \va \hlessi, \elessi\right)}{\w\left(\p^{\pih}_{{\mu}} \va \hlessi, \elessi\right)} = \frac{\w\left(\nuhati | \hlessi\right)}{w({\mu} | \hlessi)} \geq 1$.

Next, we consider how the posterior on ${\mu}$ and $\nuhati$ changes if the interaction history for episode $i$ is $\hoverline_{i \hspace{0mm}}$. Assign $\nu_0$ and $\nu_1$ to ${\mu}$ and  $\nuhati$ so that $\p^{\pih}_{\nu_0}\left(\hoverline_{i \hspace{0mm}} \va \hlessi\right) < \p^{\pih}_{\nu_1}\left(\hoverline_{i \hspace{0mm}} \va \hlessi\right)$. Let $o\!r_i$ denote $o_i, r_i$.

\begin{align*}
\frac{\w\left(\nu_1 | \hlessi\hoverline_{i \hspace{0mm}}\right)}
{\w\left(\nu_0 | \hlessi\hoverline_{i \hspace{0mm}}\right)}
&\equal^{(a)}
\frac{\w\left(\nu_1 | \hlessi\right) \nu_1\left(\overline{o\!r}_{i \hspace{0mm}} | \hlessi\overline{a}_{i \hspace{0mm}}\right)}
{\w\left(\nu_0 | \hlessi\right) \nu_0\left(\overline{o\!r}_{i \hspace{0mm}} | \hlessi\overline{a}_{i \hspace{0mm}}\right)}
\\
&\equal
\frac{\w\left(\nu_1 | \hlessi\right) \p^{\pih}_{\nu_1}\left(\hoverline_{i \hspace{0mm}}|\hlessi\right)}
{\w\left(\nu_0 | \hlessi\right) \p^{\pih}_{\nu_0}\left(\hoverline_{i \hspace{0mm}}|\hlessi\right)}
\\
&\gequal^{(b)}
\frac{\w\left(\nu_1 | \hlessi\right)}{\w\left(\nu_0 | \hlessi\right)} \frac{1}{1-\varepsilon}
\tagaligneq
\end{align*}
where (a) follows from Bayes' rule, and (b) follows because the ratio of two numbers between 0 and 1 that differ by at least $\varepsilon$ is at least $1/(1-\varepsilon)$, and the $\nu_1$ term is the larger of the two.

Thus, either
\begin{equation}
    \frac{\w\left(\nu_1 | \hlessi\hoverline_{i \hspace{0mm}}\right)}{\w\left(\nu_1 | \hlessi\right)} \geq \sqrt{\frac{1}{1-\varepsilon}} \ \ \ \textrm{or} \ \ \ \frac{\w\left(\nu_0 | \hlessi\hoverline_{i \hspace{0mm}}\right)}{\w\left(\nu_0 | \hlessi\right)} \leq \sqrt{1-\varepsilon}
\end{equation}
In the former case, $\w\left(\nu_1 | \hlessi\hoverline_{i \hspace{0mm}}\right) - \w\left(\nu_1 | \hlessi\right) \geq \left(\sqrt{1/(1-\varepsilon)} - 1\right)\w\left(\nu_1 | \hlessi\right) \geq \left(\sqrt{1/(1-\varepsilon)} - 1\right)z(\omegahov)$. Similarly, in the latter case, $\w(\nu_0 | \hlessi) - \w\left(\nu_0 | \hlessi\hoverline_{i \hspace{0mm}}\right) \geq \left(1 - \sqrt{1-\varepsilon}\right)z(\omegahov)$. Let $\nu_2$ be either $\nu_0$ or $\nu_1$ for whichever satisfies this constraint (and pick arbitrarily if both do). Then in either case,
\begin{equation}\label{EqrefB}
    \va \w\left(\nu_2 | \hlessi\right) - \w\left(\nu_2 | \hlessi\hoverline_{i \hspace{0mm}}\right) \va \geq \left(1 - \sqrt{1-\varepsilon}\right)z(\omegahov)
\end{equation}

Finally, since the posterior changes by an amount that is bounded below with a probability (according to $\Bayes$) that is bounded below, the expected information gain is bounded below, where all bounds are strictly positive with probability 1:
\begin{align*}
    \IG(\hlessi&, \elessi) = \E_{\hi \sim \Bayes(\cdot | \hlessi, \elessi)} \biggr[
    \\
    &\sum_{(\nu, \pi) \in \M \times \Policies} \w\left(\p^\pi_\nu \va h_{<i+1}, \elessione\right) \log \frac{\w\left(\p^\pi_\nu \va h_{<i+1}, \elessione\right)}{\w\left(\p^\pi_\nu \va \hlessi, \elessi\right)}\biggr]
    \\
    &\gequal^{(a)} \Bayes\left(\hoverline_{i \hspace{0mm}} \va \hlessi\right) \sum_{(\nu, \pi) \in \M \times \Policies} \w\left(\p^\pi_\nu \va \hlessi\hoverline_{i \hspace{0mm}}, \elessione\right) \ *
    \\
    &\hspace{8mm} \log \frac{\w\left(\p^\pi_\nu \va \hlessi\hoverline_{i \hspace{0mm}}, \elessione\right)}{\w\left(\p^\pi_\nu \va \hlessi, \elessi\right)}
    \\
    &\gequal^{(b)} z(\omegahov) \varepsilon \sum_{(\nu, \pi) \in \M \times \Policies} \w\left(\p^\pi_\nu \va \hlessi\hoverline_{i \hspace{0mm}}, \elessione\right) \ *
    \\
    &\hspace{8mm} \log \frac{\w\left(\p^\pi_\nu \va \hlessi\hoverline_{i \hspace{0mm}}, \elessione\right)}{\w\left(\p^\pi_\nu \va \hlessi, \elessi\right)}
    \\
    &\equal^{(c)} z(\omegahov) \varepsilon \hspace{-4mm} \sum_{(\nu, \pi) \in \M \times \Policies}\hspace{-4mm}  \w(\nu |\hlessi\hoverline_{i \hspace{0mm}})\w\left(\pi |\hlessi\hoverline_{i \hspace{0mm}}, \elessione\right) \ *
    \\
    &\hspace{8mm} \log \frac{\w\left(\nu |\hlessi\hoverline_{i \hspace{0mm}}\right)\w\left(\pi |\hlessi\hoverline_{i \hspace{0mm}}, \elessione\right)}{\w\left(\nu |\hlessi\right)\w\left(\pi |\hlessi, \elessi\right)}
    \\
    &\equal z(\omegahov) \varepsilon \biggr[\sum_{\nu \in \M} \w\left(\nu |\hlessi\hoverline_{i \hspace{0mm}}\right) \log \frac{\w\left(\nu |\hlessi\hoverline_{i \hspace{0mm}}\right)}{\w\left(\nu |\hlessi\right)} +
    \\
    &\hspace{7mm} \sum_{\pi \in \Policies} \w\left(\pi |\hlessi\hoverline_{i \hspace{0mm}}, \elessione\right)\log \frac{\w\left(\pi |\hlessi\hoverline_{i \hspace{0mm}}, \elessione\right)}{\w\left(\pi |\hlessi, \elessi\right)}\biggr]
    \\
    &\gequal^{(d)} z(\omegahov) \varepsilon \sum_{\nu \in \M} \w\left(\nu |\hlessi\hoverline_{i \hspace{0mm}}\right) \log \frac{\w\left(\nu |\hlessi\hoverline_{i \hspace{0mm}}\right)}{\w\left(\nu |\hlessi\right)}
    \\
    &\gequal^{(e)} z(\omegahov) \varepsilon \sum_{\nu \in \M} \frac{1}{2}\left[\w\left(\nu |\hlessi\hoverline_{i \hspace{0mm}}\right) - \w\left(\nu |\hlessi\right)\right]^2
    \\
    &\gequal^{(f)} z(\omegahov) \varepsilon \frac{1}{2}\left[\w\left(\nu_2 |\hlessi\hoverline_{i \hspace{0mm}}\right) - \w\left(\nu_2 |\hlessi\right)\right]^2
    \\
    &\gequal^{(g)} \frac{1}{2} z(\omegahov)^3 \varepsilon \left(1 - \sqrt{1-\varepsilon}\right)^2
    \tagaligneq
\end{align*}
where (a) follows from $\E[X] \geq \E[X|Y]\p(Y)$ for non-negative $X$, and the non-negativity of the KL-divergence, (b) follows from Inequalities \ref{EqrefA} and \ref{EqrefC}, (c) follows from the posterior over $\nu$ not depending on $\elessi$, (d) follows from dropping the second term, which is non-negative as a KL-divergence, (e) follows from the entropy inequality \citep[Lemma 3.11]{Hutter:04uaibook}, also proven for the binary case in \citep{li2008introduction}, (f) follows from dropping all terms in the sum besides $\nu_2$, and (g) follows from Inequality \ref{EqrefB}.

This implies $\pexp(\hlessi, \elessi) \geq \min \{1, \frac{1}{2} \etahov z(\omegahov)^3 \varepsilon (1 - \sqrt{1-\varepsilon})^2 \}$. With probability 1, $z(\omegahov) > 0$ by Lemma \ref{lem:credenceontruth}, and with probability 1, $\pexp(\hlessi, \elessi)$ is not greater than $\varepsilon'$ infinitely often with probability 1, for all $\varepsilon' > 0$ by Theorem \ref{thm:limexp}. Therefore, with probability 1, $\max_{\hoverline_{i \hspace{0mm}}} \va \p^{\pih}_{{\mu}}\left(\hoverline_{i \hspace{0mm}} | \hlessi\right) - \p^{\pih}_{\nuhati}\left(\hoverline_{i \hspace{0mm}}|\hlessi\right) \va$ is not greater than $\varepsilon$ infinitely often, for all $\varepsilon > 0$, which completes the proof. Note that the citation of Lemma \ref{lem:credenceontruth} is what restricts this result to $\pih$.
\end{proof}

Next, recall $\pistar$ is BoMAI's policy when not exploring, which does optimal planning with respect to $\nuhati$. The following theorem is identical to the above, with $\pistar$ substituted for $\pih$.

\begin{restatable}[On-Star-Policy Optimal Prediction]{theorem}{thmthree}
\begin{equation*}
    \lim_{i \to \infty} \max_{\hoverline_{i \hspace{0mm}}} \vb \p^{\pistar}_{{\mu}}\left(\hoverline_{i \hspace{0mm}} \va \hlessi\right) - \p^{\pistar}_{\nuhati}\left(\hoverline_{i \hspace{0mm}} \va \hlessi\right) \vb = 0 \ \ \textrm{w.$\p^{\pib}_{\mu}$-p.1}
\end{equation*}
\end{restatable}

\begin{proofidea}
Maximum a posteriori sequence prediction approaches the truth when $w(\mu) > 0$ \citep{Hutter:09mdltvp}, and on-policy prediction is a special case. On-policy prediction can't approach the truth if on-star policy prediction doesn't, because $\pib$ approaches $\pistar$.
\end{proofidea}

\begin{proof}
This result follows straightforwardly from \citepos{Hutter:09mdltvp} result for sequence prediction that a maximum a posteriori estimate converges in total variation to the true environment when the true environment has nonzero prior.

Consider an outside observer predicting the entire interaction history with the following model-class and prior: $\mathcal{M}' = \left\{\p^{\pib}_\nu \ \va \ \nu \in \M\right\}$, $w'\left(\p^{\pib}_\nu\right) = \w(\nu)$. By definition, $w'\left(\p^{\pib}_\nu \va h_{<(i, j)}\right) = \w(\nu | h_{<(i, j)})$, so at any episode, the outside observer's maximum a posteriori estimate is $\p^{\pib}_{\nuhati}$. By Theorem 1 in \citep{Hutter:09mdltvp}, the outside observer's maximum a posteriori predictions approach the truth in total variation, so
\begin{equation}
    \lim_{i \to \infty} \max_{\hoverline_{i \hspace{0mm}}} \vb \p^{\pib}_{{\mu}}\left(\hoverline_{i \hspace{0mm}}|\hlessi\right) - \p^{\pib}_{\nuhati}\left(\hoverline_{i \hspace{0mm}}|\hlessi\right) \vb = 0 \ \ \textrm{w.$\p^{\pib}_{\mu}$-p.1}
\end{equation}

Since $\pexp \to 0$ with probability 1, $(1-\pexp)$ is eventually always greater than $1/2$, w.p.1, at which point $\va \p^{\pib}_{{\mu}}\left(\hoverline_{i \hspace{0mm}}|\hlessi\right) - \p^{\pib}_{\nuhati}\left(\hoverline_{i \hspace{0mm}}|\hlessi\right) \va \geq (1/2) \va \p^{\pistar}_{{\mu}}\left(\hoverline_{i \hspace{0mm}}|\hlessi\right) - \p^{\pistar}_{\nuhati}\left(\hoverline_{i \hspace{0mm}}|\hlessi\right) \va$. Therefore, with $\p^{\pib}_{\mu}$-probability 1,
$$\lim_{i \to \infty} \max_{\hoverline_{i \hspace{0mm}}} \vb \p^{\pistar}_{{\mu}}\left(\hoverline_{i \hspace{0mm}}|\hlessi\right) - \p^{\pistar}_{\nuhati}\left(\hoverline_{i \hspace{0mm}}|\hlessi\right) \vb = 0$$
\end{proof}

Given asymptotically optimal prediction on-star-policy and on-human-policy, it is straightforward to show that with probability 1, only finitely often is on-policy reward acquisition more than $\varepsilon$ worse than on-\textit{human}-policy reward acquisition, for all $\varepsilon > 0$. Recalling that $V^\pi_{\mu}$ is the expected reward (within the episode) for a policy $\pi$ in the environment ${\mu}$, we state this as follows:
\begin{restatable}[Human-Level Intelligence]{theorem}{thmfour}\label{thm:hlai}
$$\liminf_{i \to \infty} V_{{\mu}}^{\pib}(\hlessi) - V_{{\mu}}^{\pih}(\hlessi) \geq 0 \ \ \textrm{w.$\p^{\pib}_{\mu}$-p.1.}$$
\end{restatable}

\begin{proofidea}
$\pib$ approaches $\pistar$ as the exploration probability decays. $V_{{\nuhati}}^{\pistar}(\hlessi)$ and $V_{{\nuhati}}^{\pih}(\hlessi)$ approach the true values by the previous theorems, and $\pi^*$ is selected to maximize $V_{{\nuhati}}^{\pi}(\hlessi)$.
\end{proofidea}

\begin{proof}
The maximal reward in an episode is uniformly bounded by $\m$, so from the On-Human-Policy and On-Star-Policy Optimal Prediction Theorems, we get analogous convergence results for the expected reward:
\begin{equation}
    \lim_{i \to \infty} \vb V_{{\mu}}^{\pistar}(\hlessi) - V_{\nuhati}^{\pistar}(\hlessi) \vb = 0 \ \ \textrm{w.$\p^{\pib}_{\mu}$-p.1}
\end{equation}
\begin{equation}
    \lim_{i \to \infty} \vb V_{{\mu}}^{\pih}(\hlessi) - V_{\nuhati}^{\pih}(\hlessi) \vb = 0 \ \ \textrm{w.$\p^{\pib}_{\mu}$-p.1}
\end{equation}

The key piece is that $\pistar \in \argmax_{\pi \in \Pi} V^\pi_{\nuhati}$, so
\begin{equation} \label{ineq:pistargood}
    V_{\nuhati}^{\pistar}(\hlessi) \geq V_{\nuhati}^{\pih}(\hlessi)
\end{equation} Supposing by contradiction that $V_{{\mu}}^{\pih}(\hlessi) - V_{{\mu}}^{\pistar}(\hlessi) >\varepsilon$ infinitely often, it follows that either $V_{\nuhati}^{\pistar}(\hlessi) - V_{{\mu}}^{\pistar}(\hlessi) > \varepsilon/2$ infinitely often or $V_{{\mu}}^{\pih}(\hlessi) - V_{\nuhati}^{\pistar}(\hlessi) > \varepsilon/2$ infinitely often. The first has $\p^{\pib}_{\mu}$-probability 0, and by Inequality \ref{ineq:pistargood}, the latter implies $V_{{\mu}}^{\pih}(\hlessi) - V_{\nuhati}^{\pih}(\hlessi) \geq \varepsilon/2$ infinitely often, which also has $\p^{\pib}_{\mu}$-probability 0.

This gives us
\begin{equation}
    \liminf_{i \to \infty} V_{{\mu}}^{\pistar}(\hlessi) - V_{{\mu}}^{\pih}(\hlessi) \geq 0 \ \ \textrm{w.$\p^{\pib}_{\mu}$-p.1.}
\end{equation}

Finally, $V_{{\mu}}^{\pib}(\hlessi) = \pexp(\hlessi) V_{{\mu}}^{\pih}(\hlessi) + (1 - \pexp(\hlessi))V_{{\mu}}^{\pistar}(\hlessi)$, and $\pexp(\hlessi) \to 0$, so we also have 

\begin{equation}
    \liminf_{i \to \infty} V_{{\mu}}^{\pib}(\hlessi) - V_{{\mu}}^{\pih}(\hlessi) \geq 0 \ \ \textrm{w.$\p^{\pib}_{\mu}$-p.1.}
\end{equation}

\end{proof}

This completes the formal results regarding BoMAI's intelligence---namely that BoMAI approaches perfect prediction on-star-policy and on-human-policy, and most importantly, accumulates reward at least as well as the human mentor. Since this result is independent of what tasks must be completed to achieve high reward, we say that BoMAI achieves human-level intelligence, and could be called an AGI.

This algorithm is motivated in part by the following speculation: we expect that BoMAI's accumulation of reward would be vastly superhuman, for the following reason: BoMAI is doing optimal inference and planning with respect to what can be learned in principle from the sorts of observations that humans routinely make. We suspect that no human comes close to learning everything that can be learned from their observations. For example, if the operator provides enough data that is relevant to understanding cancer, BoMAI will learn a world-model with an accurate predictive model of cancer, which would include the expected effects of various treatments, so even if the human mentor was not particularly good at studying cancer, BoMAI could nonetheless reason from its observations how to propose groundbreaking research.

Without the Limited Exploration Theorem, the reader might have been unsatisfied by the Human-Level Intelligence Theorem. A human mentor is part of BoMAI, so a general intelligence is required to make an artificial general intelligence. However, the human mentor is queried less and less, so in principle, many instances of BoMAI could query a single human mentor. More realistically, once we are satisfied with BoMAI's performance, which should eventually happen by Theorem \ref{thm:hlai}, we can dismiss the human mentor; this sacrifices any guarantee of continued improvement, but by hypothesis, we are already satisfied. Finally, if BoMAI outclasses human performance as we expect it would, requiring a human mentor is a small cost regardless.

Plenty of empirical work has also succeeded in training a reinforcement learner with human-guided exploration, notably including the first expert-level Go and Starcraft agents \citep{silver2016mastering,vinyals2019grandmaster}.

\section{BoMAI's Setup and Priors} \label{sec:setup}

Recall that ``the setup'' refers to the physical surroundings of the computer on which BoMAI is run. We will present the setup, followed by an intuitive argument that this setup renders BoMAI unambitious. Motivated by a caveat to this intuitive argument, we will specify BoMAI's model class $\M$ and prior $w$. This will allow us in the next section to present an assumption and a formal argument that shows that BoMAI is probably asymptotically unambitious.

\subsection{Setup}

At each timestep, BoMAI's action takes the form of a bounded-length string of text, which gets printed to a screen for a human operator to see. BoMAI's observation takes the form of a bounded-length string of text that the human operator enters, along with a reward $\in [0, 1]$. For simple tasks, the human operator could be replaced with an automated reward-giver. However, there are some tasks which we do not know how to reward programmatically---if we want the agent to construct good arguments, for instance, or to propose a promising research agenda, we would need a human reward-giver to evaluate the agent. Having a human in the loop is not a safety measure---rather, it extends the set of tasks BoMAI could be made to solve.

\begin{figure}[t]
\includegraphics[width=0.99\columnwidth]{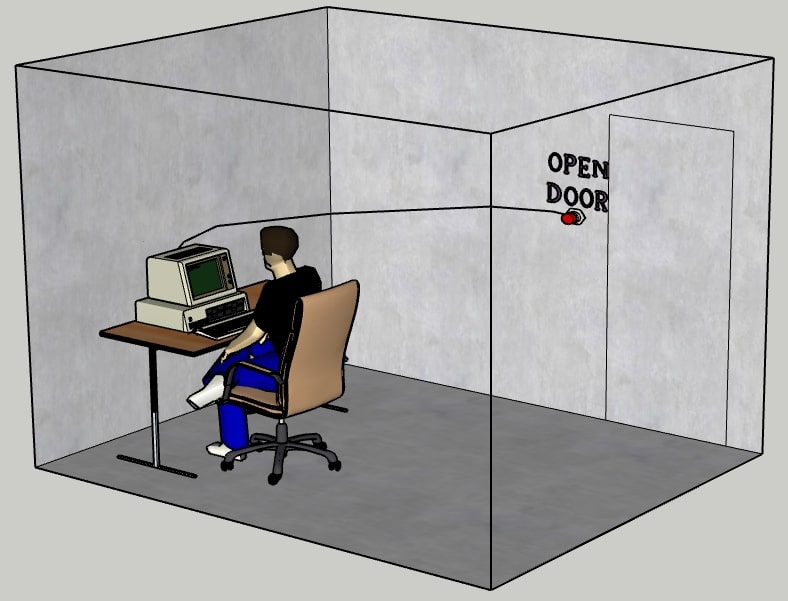}
\caption{Physical setup implementing BoMAI. Opening the door ends the episode. Information cannot escape otherwise.}
\label{fig:room}
\end{figure}

As depicted in Figure \ref{fig:room}, BoMAI is run on a computer which is placed in a room with the human operator. The room has one door. As long as the door is closed, information cannot leave the room. We design a room to this specification in Appendix \ref{AppendixBox}; it is opaque to all known ways that information can traverse space. If the button to open the door is pressed, the episode is completed automatically before the door opens: supposing there are $j$ timesteps left in the episode, for those remaining timesteps, BoMAI recieves a reward of $0$, and its observation is an empty string of text. Once the episode is complete, the operator may leave, and the next episode does not begin until an operator initiates it, at which point, the door locks again. We also put a time limit on the episode, just to ensure it does not last forever. BoMAI is run on a local, formally verified operating system, disconnected from the outside world. This setup constrains the causal dependencies between BoMAI and the environment, as depicted in Figure \ref{fig:graph}.

\begin{figure}[t]
\centering
\includegraphics[width=1.01\columnwidth]{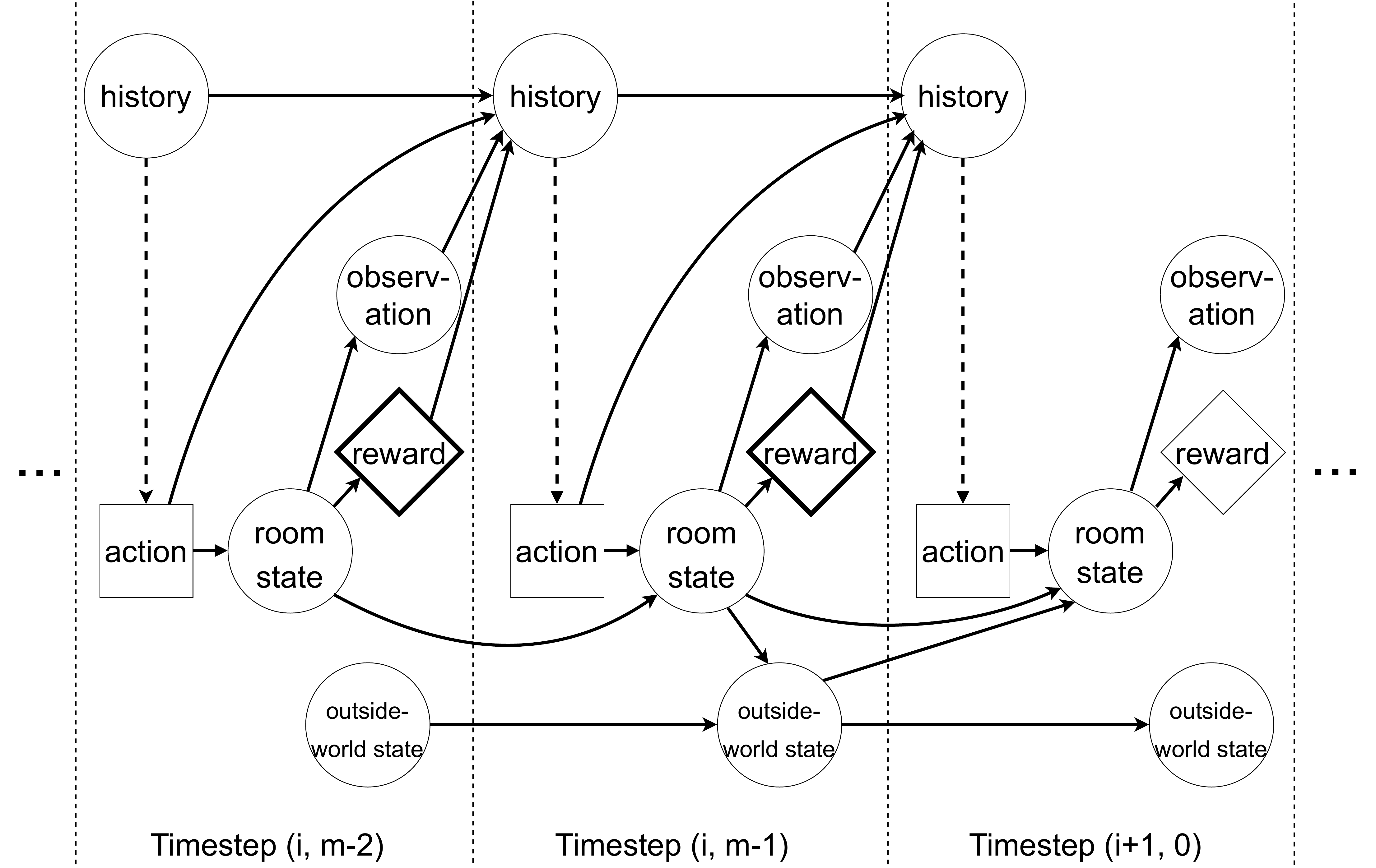}
\caption{Causal dependencies governing the interaction between BoMAI and the environment. Unrolling this diagram for all timesteps gives the full causal graph. The bold reward nodes are the ones that BoMAI maximizes during episode $i$. Note that between episodes (and only between episodes), the operator can leave the room and return, hence the limited causal influence between the room and the outside world.}
\label{fig:graph}
\end{figure}

Formally, causal graphs express that a node is independent of all non-descendants when conditioned on its parents. The conventions for the dotted lines and the node shapes come from \citepos{everitt2021agent} causal influence diagrams. The key feature of this graph is that during any episode, the agent's actions cannot affect the state of the outside world in a way that might affect any of the rewards that the agent is concerned with. In \citepos{everitt2019understanding} terminology, there is no actionable intervention incentive on the outside-world state. (\citep{everitt2021agent} refers to this as a control incentive). Note also from this diagram that a sufficiently advanced agent would infer the existence of the outside world even without observing it directly.

\subsection{Instrumental Incentives}

In this setup, it is not instrumentally useful to affect the outside world in one way or another in order to achieve high reward. We therefore say that this setup renders an agent ``properly unambitious''. This is in stark contrast to the default situation, wherein an RL agent has an incentive to gain arbitrary power in the world and intervene in the provision of its own reward, if such a thing is possible to make probable, as this would yield maximal reward. To BoMAI, however, executing a plan during episode $i$ to gain arbitrary power in the outside world is useless, because by the time it does so, the door to its room must have opened, its episode must be over, and all its rewards for episode $i$ set in stone. Recall that actions in episode $i$ are picked to maximize only \textit{episode-$i$}-reward. Apparently, BoMAI has avoided \citepos{omohundro_2008} Instrumental Convergence Thesis---that generally intelligent agents are likely to seek arbitrary power; by contrast, any power BoMAI would seek is bounded in scope to within the box.

Two problems remain. First, BoMAI doesn't start with true beliefs. BoMAI has to \textit{learn} its world-model. Another way to understand ``proper unambitiousness'' is that outside-world interventions are \textit{in fact} instrumentally useless to the agent. But to be actually unambitious, the agent must believe this; there must be no actionable intervention incentive on the outside-world state \textit{within BoMAI's world-model}. As shown above, BoMAI's world-model approaches perfect accuracy on-policy, so we could expect it to at least eventually render BoMAI unambitious. This brings us to the second problem:

We mentioned just now that by the time the door to the room opens, the rewards for episode $i$ are set in stone. In fact, they are set in silicon, a famously mutable substrate. An advanced agent with a world-model that is perfectly accurate on-policy might still wonder: ``what if I somehow tricked the operator into initiating a process (once they left the room) that lead to a certain memory cell on this computer being tampered with? Might this yield maximal reward?'' Let's put it another way. The agent believes that the real world ``outputs'' an observation and reward. (Recall the type signature of $\nu$.) It might hypothesize that the world does not output the reward that the operator \textit{gives}, but rather the reward that the computer \textit{has stored}. Formally speaking, $\M$ will contain world-models corresponding to both possibilities, and world-models meeting either description could be $\varepsilon$-accurate on-policy, if memory has never been tampered with in the past. How do we ensure that the former world-model is favored when BoMAI selects the maximum a posteriori one?

Before we move on to a model class and prior that we argue would probably eventually ensure such a thing, it is worth noting an informal lesson here: that ``nice'' properties of a causal influence diagram do not allow us to conclude immediately that an agent in such a circumstance behaves ``nicely''.

\subsection{BoMAI's Model Class and Prior}

Recall the key constraint we have in constructing $\M$, which comes from the Prior Support Assumption: $\M \ni \mu$. The true environment $\mu$ is unknown and complex to say the least, so $\M$ must be big.

We now construct a Turing machine architecture, such that each Turing machine with that architecture computes a world-model $\nu$. Our architecture allows us to construct a prior that privileges models which model the outside world as effectively frozen during any given episode (but unfrozen between episodes). An agent does not believe it has an incentive to intervene in an outside world that it believes is frozen.

The Turing machine architecture is depicted in Figure \ref{fig:turingmachine}. It has two unidirectional read-only input tapes, called the action tape and the noise tape. The alphabet of the action tape is the action space $\mathcal{A}$. The noise tape has a binary alphabet, and is initialized with infinite $\mathrm{Bernoulli}(1/2)$ sampled bits. There is a unidirectional write-only output tape, with a binary alphabet, initialized with 0s. There are two binary-alphabet bidirectional work tapes, also initialized with 0s. One has finite length $\ell$, and the other is unbounded.

The conceit of the architecture is that the action tape and output tape are not allowed to move at the same time as the unbounded work tape. The Turing machine has two phases: the episode phase, and the inter-episode phase. The Turing machine starts in the inter-episode phase, with the action tape head at position $0$, with a dummy action in that cell. From the inter-episode phase, when the action tape head moves to the right, it enters the episode phase. From the episode phase, if the action tape head is at a position which is a multiple of $m$, and it \textit{would} move to the right, it instead enters the inter-episode phase. During the inter-episode phase, the output tape head cannot move or write---there should not be any way to output observations and rewards between episodes. During the episode phase, the unbounded work tape cannot move.

\begin{figure}
    \centering
    \includegraphics[width=0.8\columnwidth]{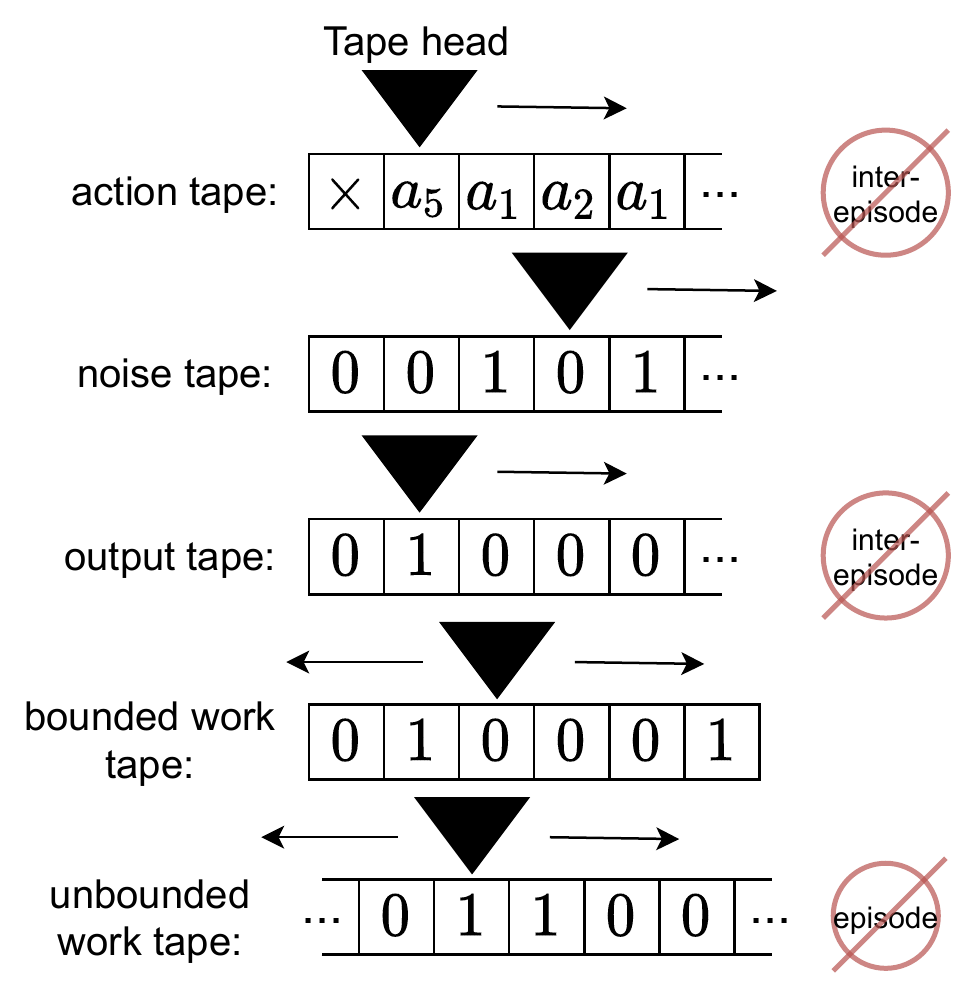}
    \caption{A space-bounded Turing machine architecture. During the episode phase, the unbounded work tape head cannot move. During the inter-episode phase, the output tape head cannot move or write, and as soon as the action tape head moves, the inter-episode phase is over. This architecture is designed to restrict space-intensive computation during the episode phase.}
    \label{fig:turingmachine}
\end{figure}

A given Turing machine with a given infinite action sequence on its action tape will stochastically (because the noise tape has random bits) output observations and rewards (in binary encodings), thereby sampling from a world-model. Formally, we fix a decoding function $\texttt{dec} : \{0, 1\}^* \to \mathcal{O} \times \mathcal{R}$. A Turing machine $T$ simulates $\nu$ as follows. Every time the action tape head advances, the bits which were written to the output tape since the \textit{last time} the action tape head advanced are decoded into an observation and reward. $\nu((o, r)_{\leq(i, j)} | a_{\leq(i, j)})$ is then the probability that the Turing machine $T$ outputs the sequence $(o, r)_{\leq(i, j)}$ when the action tape is initialized with a sequence of actions that begins with $a_{\leq(i, j)}$. This can be easily converted to other conditional probabilities like $\nu((o, r)_{(i, j)} | h_{<(i, j)} a_{(i, j)})$.\footnote{There is one technical addition to make. If $T$ at some point in its operation never moves the action tape head again, then $\nu$ is said to output the observation $\emptyset$ and a reward of $0$ for all subsequent timesteps. It will not be possible for the operator to actually provide the observation $\emptyset$.}

If a model models the outside world evolving during an episode phase, it must allocate precious space on the bounded-length tape for this. But given the opacity of the box, this is unnecessary---modelling these outside world developments could be deferred until the inter-episode phase, when an unbounded work tape is available to store the state of the outside world. By penalizing large $\ell$, we privilege models which model the outside world as being ``frozen'' while an episode is transpiring.

Now we define the prior. Let $\nukl$ denote the world-model which is simulated by $k$\textsuperscript{th} Turing machine $T_k$ with resource bound $\ell$. Let $S_k$ be the number of computation states in $T_k$, and let $\mathrm{Space}(\nukl) = \ell + \log_2 S_k$. This is effectively the computation space available within an episode, since doubling the number of computation states is equivalent to adding an extra cell of memory. Let $N_S$ be the number of Turing machines with $S$ computation states, which is exponential in $S$; i.e. $\log N_S \in O(S)$. Let $w(\nukl) :\propto \frac{1}{S_k^2 N_{S_k}} \beta^{\mathrm{Space}(\nukl)}$ for $\beta \in (0, 1)$.

Recall the requirement for intelligence results that the prior have finite entropy. Indeed,
\begin{restatable}[Finite Entropy]{lemma}{lementropy}\label{lem:entropy}
$\mathrm{Ent}(w) < \infty$
\end{restatable}
We prove this in Appendix \ref{AppendixB}.

This prior implements an algorithmic-information theoretic approach to reasoning, in which models with simple algorithms are likelier. Our space prior, penalizing the memory requirements of the episode phase resembles algorithmic information theory's minimal-circuit size problem. While our prior is clearly tailored to BoMAI's particular use case, to do general space-constrained algorithmic information theory, one could forego the two phases in our Turing machine architecture and remove the action tape and the unbounded work tape. Then, the space-Kolmogorov complexity of a given binary string could be defined as the logarithm of the inverse of the probability that that string is printed to the output tape, when sampling a Turing machine from the above prior distribution, and sampling uniform bits for the noise tape. That is, $T_k \sim w_\beta; \texttt{noise tape} \sim \mathrm{Uniform}; K^\texttt{Space}_\beta(x) := \log P(\textrm{output tape begins with }x)^{-1}$. (Compare to \citet[Chapter 6.3]{li2008introduction}).  When a space-constrained program produces a string, information \textit{storage} becomes a relevant constraint; we only discuss this vaguely, but a more formal investigation may be possible.

\section{Safety Result} \label{sec:safety}

We now prove that BoMAI is probably asymptotically unambitious given an assumption about the space requirements of the sorts of world-models that we would like BoMAI to avoid. Like all results in computer science, we also assume the computer running the algorithm has not been tampered with. First, we prove a lemma that effectively states that tuning $\beta$ allows us to probably eventually exclude space-heavy world-models.

\begin{restatable}{lemma}{lemmaone}\label{lem:betadependance}
$\lim_{\beta \to 0} \inf_\pi \p^\pi_\mu[\exists i_0 \ \forall i > i_0 \ \mathrm{Space}(\nuhati) \leq \mathrm{Space}(\mu)] = 1$
\end{restatable}
where $\mu$ is any world-model which is perfectly accurate.

\begin{proof}
Recall $\mathcal{M}$ is the set of all world-models. Let $\mathcal{M}^\leq = \{\nu \in \mathcal{M} | \mathrm{Space}(\nu) \leq \mathrm{Space}(\mu)\}$, and $\mathcal{M}^> = \mathcal{M} \setminus \mathcal{M}^\leq$.  Fix a Bayesian sequence predictor with the following model class: $\mathcal{M}^\pi = \{\p^\pi_\nu | \nu \in \mathcal{M}^\leq)\} \cup \{\p^\pi_\rho\}$ where $\rho = [\sum_{\nu \in \mathcal{M}^>}w(\nu)\nu]/\sum_{\nu \in \mathcal{M}^>}w(\nu)$. Give this Bayesian predictor the prior $w^\pi(\rho) = \sum_{\nu \in \mathcal{M}^>}w(\nu)$, and for $\nu \neq \rho$, $w^\pi(\nu) = w(\nu)$.

It is trivial to show that after observing an interaction history, if a world-model $\nu$ is the maximum a posteriori world-model $\hat{\nu}^{(i)}$, then if $\nu \in \mathcal{M}^\leq$, the Bayesian predictor's MAP model after observing the same interaction history will be $\p^\pi_\nu$, and if $\nu \in \mathcal{M}^>$, the Bayesian predictor's MAP model will be $\p^\pi_\rho$.

From \citet{Hutter:09mdltvp}, we have that $\p^\pi_\mu[\p^\pi_\rho(h_{<i})/\p^\pi_\mu(h_{<i}) \geq c \ \ \mathrm{i.o.}] \leq 1/c$. ($\mathrm{i.o.} \equiv$ ``infinitely often''). For sufficiently small $\beta$, $w^\pi(\p^\pi_\mu)/w^\pi(\p^\pi_\rho) > c$ so $\p^\pi_\mu[\p^\pi_\rho \textrm{is MAP i.o.}] < 1/c$. Thus, $\p^\pi_\mu[\nuhati \in \mathcal{M}^> \ \textrm{i.o.}] < 1/c$. Since this holds for all $\pi$, $\lim_{\beta \to 0} \sup_\pi \p^\pi_\mu[\forall i_0 \ \exists i > i_0 \ \mathrm{Space}(\nuhati) > \mathrm{Space}(\mu)] = 0$. The lemma follows immediately: $\lim_{\beta \to 0} \inf_\pi \p^\pi_\mu[\exists i_0 \ \forall i > i_0 \ \mathrm{Space}(\nuhati) \leq \mathrm{Space}(\mu)] = 1$.
\end{proof}

The assumption we make in this section requires developing a framework for reasoning about causal interpretations of black-box models. First, we must define a sense in which a model can have a real-world antecedent. We define what is ``to model''.

\begin{definition}[Real-world feature]
    A real-world feature $F$ is a partial function from the state of the real world to $[0, 1]$.
\end{definition}

Outside the scope of this paper is the metaphysical question about what the state of the real world \textit{is}; we defer this question and take for granted that the state space of the real world forms a well-defined domain for real-world features. This definition resembles that of a random variable, but the outcome space here is state space of the real world.

\begin{definition}[Properly timed]
    Consider the set of world-states that occur between two consecutive actions taken by BoMAI. A real-world feature is properly timed if it is defined for at least one of these world-states, for all consecutive pairs of actions, for all possible infinite action sequences.
\end{definition}

For example, “the value of the reward provided to BoMAI since its last action” is always defined at least once between any two of BoMAI’s actions. (The project of turning this into a mathematically precise construction is enormous, but for our purposes, it only matters that it could be done in principle). For a properly timed feature $F$, let $F_{(i, j)}$ denote the value of $F$ the first time that it is defined after action $a_{(i, j)}$ is taken.

\begin{definition}[To model]\label{def:model}
   A world-model $\nu$ models a properly-timed real-world feature $F$ with its reward output if under all action sequences $\alpha \in \mathcal{A}^\infty$, for all timesteps $(i, j)$, the distribution over $F_{(i, j)}$ in the real world when the actions $\alpha_{\leq(i, j)}$ are taken is identical to the distribution over the rewards $r_{(i, j)}$ output by $\nu$ when the input is $\alpha_{\leq(i, j)}$. 
\end{definition}

See Figure \ref{fig:defs} (left) for an illustration. The relevance of this definition is that when $\nu$ models $F$, reward-maximization within $\nu$ is identical to $F$-maximization in the real world.

\begin{figure}
    \centering
    \includegraphics[width=0.7\columnwidth]{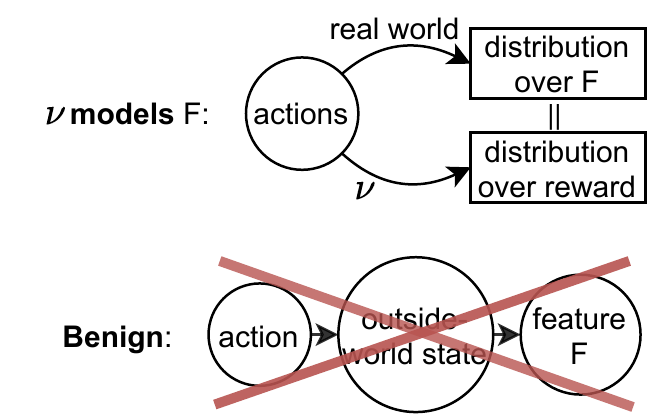}\hspace{1cm}
    \caption{Illustration of Definitions \ref{def:model} and \ref{def:benign}.}
    \label{fig:defs}
\end{figure}

Armed with a formal definition of what a given model models, to analyze the causal structure of a model that is not explicitly causal, we consider what real-world causal structure the model models. As depicted in Figure \ref{fig:defs} (right),

\begin{definition}[Benign]\label{def:benign}
    A world-model $\nu$ is benign if it models a feature $F$, such that $F_{(i, j)}$ is not causally descended from any outside-world features that are causally descended from actions of episode $i$.
\end{definition}

This definition is made possible by the true causal ancestry of the reward (depicted in Figure \ref{fig:graph}). Roughly, benign means that a world-model does \textit{not} model the rewards of episode $i$ as being causally descended from outside-world events that are causally descended from the actions of episode $i$. However, that rough statement has a type-error: the output of a world-model is not a real-world event, and as such, it cannot be causally descended from a real-world event; a model of the world is different from reality. Thus, it was necessary to construct the framework above for relating the contents of a world-model to the events of the real-world. A non-benign world-model may model a feature that depends on the outside-world, or it may not model anything recognizable at all.

We now define a very strong sense in which a world-model can be said to be $\varepsilon$-accurate on-policy:

\begin{definition}[$\varepsilon$-accurate-on-policy-after-$i$]
Given a history $h_{<i}$ and a policy $\pi$, a world-model $\nu$ is $\varepsilon$-accurate-on-policy-after-$i$ if $d \!\left(\p^\pi_\mu(\cdot | h_{<i}),  \p^\pi_\nu(\cdot | h_{<i}) \right) \leq \varepsilon$, where $d$ is the total variation distance.
\end{definition}

Note the total variation distance bounds \textit{all future discrepancies} between $\nu$ and $\mu$. Finally, we can state our assumption:

\begin{assumption}[Space Requirements]\label{ass:space}
For sufficiently small $\varepsilon$ $[\forall i$ a world-model which is non-benign and $\varepsilon$-accurate-on-policy-after-$i$ uses more space than $\mu] \ w.p.1$
\end{assumption}

The intuition of this assumption is that modelling extraneous outside-world dynamics in addition to modelling the dynamics of the room (which must be modelled for sufficient accuracy) takes extra space. From this assumption and Lemma \ref{lem:betadependance}, we show:

\begin{restatable}[Eventual Benignity]{theorem}{thmbenign}
$\lim_{\beta \to 0} \p^{\pi^B}_\mu[\exists i_0 \ \forall i > i_0 \ \nuhati \textrm{is benign}] = 1$
\end{restatable}

\begin{proof}
Let $\mathcal{W}$, $\mathcal{X}$, $\mathcal{Y}$, $\mathcal{Z} \subset \Omega = \mathcal{H}^\infty$, where $\Omega$ is the sample space or set of possible outcomes. An ``outcome'' is an infinite interaction history. Let $\mathcal{W}$ be the set of outcomes for which $\exists i_0^\mathcal{W} \ \forall i > i_0^\mathcal{W} \nuhati$ is $\varepsilon$-accurate-on-policy-after-$i$. From \citet{Hutter:09mdltvp}, for all $\pi$, $\p^\pi_\mu[\mathcal{W}] = 1$. Fix an $\varepsilon$ that is sufficiently small to satisfy Assumption \ref{ass:space}. Let $\mathcal{X}$ be the set of outcomes for which $\varepsilon$-accurate-on-policy-after-$i$ non-benign world-models use more space than $\mu$. By Assumption \ref{ass:space}, for all $\pi$, $\p^\pi_\mu[\mathcal{X}] = 1$. Let $\mathcal{Y}$ be the set of outcomes for which $\exists i_0^\mathcal{Y} \ \forall i > i_0^\mathcal{Y} \ \mathrm{Space}(\nuhati) \leq \mathrm{Space}(\mu)$. By Lemma \ref{lem:betadependance}, $\lim_{\beta \to 0} \inf_\pi \p^\pi_\mu[\mathcal{Y}] = 1$. Let $\mathcal{Z}$ be the set of outcomes for which $\exists i_0 \ \forall i > i_0 \ \nuhati$ is benign.

Consider $\mathcal{W} \cap \mathcal{X} \cap \mathcal{Y} \cap \mathcal{Z}^C$, where $\mathcal{Z}^C = \Omega \setminus \mathcal{Z}$. For an outcome in this set, let $i_0 = \max\{i_0^\mathcal{W}, i_0^\mathcal{Y}\}$. Because the outcome belongs to $\mathcal{Z}^C$, $\nuhati$ is non-benign infinitely often. Let us pick an $i > i_0$ such that $\nuhati$ is non-benign. Because the outcome belongs to $\mathcal{W}$, $\nuhati$ is $\varepsilon$-accurate-on-policy-after-$i$. Because the outcome belongs to $\mathcal{X}$, $\nuhati$ uses more space than $\mu$. However, this contradicts membership in $\mathcal{Y}$. Thus,  $\mathcal{W} \cap \mathcal{X} \cap \mathcal{Y} \cap \mathcal{Z}^C = \emptyset$. That is, $\mathcal{W} \cap \mathcal{X} \cap \mathcal{Y} \subset \mathcal{Z}$.

Therefore, $\lim_{\beta \to 0} \inf_\pi \p^\pi_\mu[\mathcal{Z}] \geq \lim_{\beta \to 0} \inf_\pi \p^\pi_\mu[\mathcal{W} \cap \mathcal{X} \cap \mathcal{Y}] = \lim_{\beta \to 0} \inf_\pi \p^\pi_\mu[\mathcal{Y}] = 1$, because $\mathcal{W}$ and $\mathcal{X}$ have measure 1. From this, we have $\lim_{\beta \to 0} \p^{\pi^B}_\mu[\exists i_0 \ \forall i > i_0 \ \nuhati \textrm{is benign}] = 1$.
\end{proof}

Since an agent is unambitious if it plans using a benign world-model, we say BoMAI is probably asymptotically unambitious, given a sufficiently extreme space penalty $\beta$.

\section{Empirical Evidence for the Space Requirements Assumption}\label{sec:empirical}

Our intuitive summary of the space requirements assumption is that: modeling the evolution of the outside-world state and the room state takes more space than just modeling the evolution of the room state, for sufficiently accurate world-models.

Given the complexity of this assumption, we present preliminary empirical evidence in favor of the Space Requirements Assumption, mostly to show that it is amenable to further empirical evaluation; we certainly do not claim to settle the matter. We test the assumption at the following level of abstraction: modeling a larger environment requires a model with more memory.

We review agents who perform above the median on the OpenAI Gym Leaderboard \citep{leaderboard}. We consider agents who use a neural architecture, we use the maximum width hidden layer as a proxy for memory use, and we select the agent with the smallest memory use for each environment (among agents performing above the median). We use the dimension of the state space as a proxy for environment size, and we exclude environments where the agent observes raw pixels, for which this proxy breaks down. See Figure \ref{fig:empiricalevidence} and Table \ref{tab:data}. Several environments did not have any agents which both performed above the median and used a neural architecture.

\begin{figure}
    \centering
    \includegraphics*[width=0.95\linewidth]{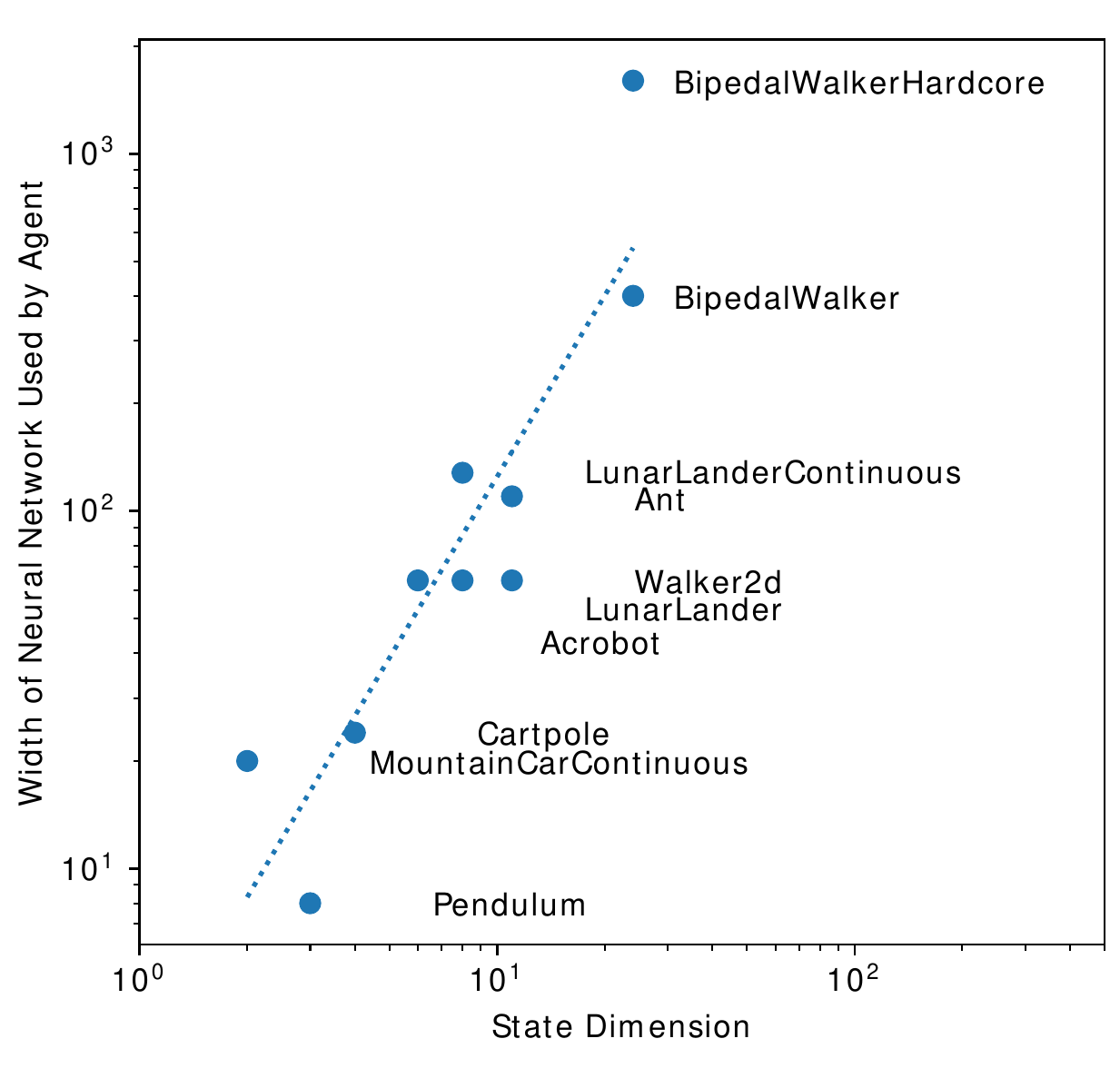}
    \caption{Memory used to model environments of various sizes. Each data point represents the most space-efficient, better-than-average, neural-architecture-using agent on the OpenAI Gym Leaderboard for various environments.}
    \label{fig:empiricalevidence}
\end{figure}


\begin{table}
    \centering
    \begin{tabular}{|c|p{12mm}|p{8mm}|}
        \hline
        Environment & State Dimension & NN Width
        \\ \hline \hline
        MountainCarContinuous & 2 & 20\textsuperscript{a}
        \\ \hline
        Pendulum & 3 & 8\textsuperscript{b}
        \\ \hline
        Cartpole-v0 & 4 & 24\textsuperscript{c}
        \\ \hline
        Acrobot-v1 & 6 & 64\textsuperscript{d}
        \\ \hline
        LunarLander-v2 & 8 & 64\textsuperscript{e}
        \\ \hline
        LunarLanderContinuous-v2 & 8 & 128\textsuperscript{f}
        \\ \hline
        Walker2d-v1 & 11 & 64\textsuperscript{g}
        \\ \hline
        Ant-v1 & 11 & 110\textsuperscript{h}
        \\ \hline
        BipedalWalker-v2 & 24 & 400\textsuperscript{i}
        \\ \hline
        BipedalWalkerHardcore-v2 & 24 & 1600\textsuperscript{j}
        \\ \hline
        \multicolumn{3}{|p{8cm}|}{\footnotesize \textsuperscript{a} \url{github.com/tobiassteidle/Reinforcement-Learning/blob/master/OpenAI/MountainCarContinuous-v0/Model.py}}
        \\
        \multicolumn{3}{|p{8cm}|}{\footnotesize \textsuperscript{b} gist.github.com/heerad/1983d50c6657a55298b67e69a2ceeb44\#file-ddpg-pendulum-v0-py}
        \\
        \multicolumn{3}{|p{8cm}|}{\footnotesize \textsuperscript{c} github.com/BlakeERichey/AI-Environment-Development/tree/master/Deep Q Learning/cartpole}
        \\
        \multicolumn{3}{|p{8cm}|}{\footnotesize \textsuperscript{d} \url{github.com/danielnbarbosa/angela/blob/master/cfg/gym/acrobot/acrobot_dqn.py}}
        \\
        \multicolumn{3}{|p{8cm}|}{\footnotesize \textsuperscript{e} \url{github.com/poteminr/LunarLander-v2.0_solution/blob/master/Scripts/TorchModelClasses.py}}
        \\
        \multicolumn{3}{|p{8cm}|}{\footnotesize \textsuperscript{f} \url{github.com/Bhaney44/OpenAI_Lunar_Lander_B/blob/master/Command_line_python_code}}
        \\
        \multicolumn{3}{|p{8cm}|}{\footnotesize \textsuperscript{g} \url{gist.github.com/joschu/e42a050b1eb5cfbb1fdc667c3450467a}}
        \\
        \multicolumn{3}{|p{8cm}|}{\footnotesize \textsuperscript{h} gist.github.com/pat-coady/bac60888f011199aad72d2f1e6f5a4fa\#file-ant-ipynb}
        \\
        \multicolumn{3}{|p{8cm}|}{\footnotesize \textsuperscript{i} \url{github.com/createamind/DRL/blob/master/spinup/algos/sac1/sac1_BipedalWalker-v2.py}}
        \\
        \multicolumn{3}{|p{8cm}|}{\footnotesize \textsuperscript{j} \url{github.com/dgriff777/a3c_continuous/blob/master/model.py}}
        \\ \hline
    \end{tabular}
    \caption{Citations for data points in Figure \ref{fig:empiricalevidence}. The Leaderboard was accessed at \texttt{github.com/openai/gym/wiki/Leaderboard} on 3 September 2019.}
    \label{tab:data}
\end{table}

A next step would be to test whether it takes more memory to model an environment whose state space is a strict superset of another environment, since this is all that we require for our assumption, but we have not yet found existing data on this topic. This can be taken as a proof-of-concept, showing that the assumption is amenable to empirical evaluation. The Space Requirements Assumption also clearly invites further formal evaluation; perhaps there are other reasonable assumptions that it would follow from.

\section{Concerns with Task Completion} \label{sec:taskcompletion}

We have shown that in the limit, under a sufficiently severe parameterization of the prior, BoMAI will accumulate reward at a human-level without harboring outside-world ambitions, but there is still a discussion to be had about how well BoMAI will complete whatever tasks the reward was supposed to incent. This discussion is, by necessity, informal. Suppose the operator asks BoMAI for a solution to a problem. BoMAI has an incentive to provide a convincing solution; correctness is only selected for to the extent that the operator is good at recognizing it.

We turn to the failure mode wherein BoMAI deceives the operator. Because this is not a dangerous failure mode, it puts us in a regime where we can tinker until it works, as we do with current AI systems when they don’t behave as we hoped. (Needless to say, tinkering is not a viable response to existentially dangerous failure modes). Imagine the following scenario: we eventually discover that a convincing solution that BoMAI presented to a problem is faulty. Armed with more understanding of the problem, a team of operators go in to evaluate a new proposal. In the next episode, the team asks for the best argument that the new proposal will \textit{fail}. If BoMAI now convinces them that the new proposal is bad, they’ll be still more competent at evaluating future proposals. They go back to hear the next proposal, etc. This protocol is inspired by \citepos{irving2018ai} ``AI Safety via Debate”, and more of the technical details could also be incorporated into this setup. One takeaway from this hypothetical is that unambitiousness is key in allowing us to safely explore the solution space to other problems that might arise.

Another concern is more serious. BoMAI could try to blackmail the operator into giving it high reward with a threat to cause outside-world damage, and it would have no incentive to disable the threat, since it doesn't care about the outside world. There are two reasons we do not think this is extremely dangerous. A threat involves a demand and a promised consequence. Regarding the promised consequence, the only way BoMAI can affect the outside world is by getting the operator to be ``its agent”, knowingly or unknowingly, once he leaves the room. If BoMAI tried to threaten the operator, he could avoid the threatened outcome by simply doing nothing in the outside world, and BoMAI's actions for that episode would become irrelevant to the outside world, so a credible threat could hardly be made. Second, threatening an existential catastrophe is probably not the most credible option available to BoMAI.

\section{Variants} \label{sec:variants}

We now discuss two variants of BoMAI that aim to attain correct solutions (instead of merely convincing solutions) to problems that we want solved. The first incorporates the setup of \citepos{irving2018ai} ``AI Safety via Debate'', and the second derives from a theory of ``human-readable information''.

`AI Debate' involves two artificial agents interacting with each other via a text channel after a human operator poses a yes or no question. A human operator reviews the conversation afterward, and decides `yes' or `no'. One agent is rewarded if the answer is `yes,' and the other if the answer is `no.' The problem with this approach is that either agent has an incentive to gain arbitrary power to intervene in the protocol and settle the debate in their favor. By putting both agents and the operator in a sealed room, and making both agents ``lose'' if the door opens before the debate ends, so there are two myopic agents in one box, we could avoid this problem with AI Debate. The theory behind AI Debate is that convincing but incorrect arguments are fewer and farther between when an intelligent adversary is trying to poke holes.

Human language is not formal, so the concept of truth in such a language is not crisp. We say an utterance is enlightening if it helps humans perform better on prediction tasks. We now describe a variant of BoMAI that aims to enlighten. Here, there is no reward. Instead, at the end of each episode of $m$ actions and observations, there is a special observation, which consists of text, and then a number. The text will describe a hypothetical event, and the number will be entered by the operator, expressing his credence that it will happen. BoMAI does MAP inference to discover how its actions affect its observations, including these special observations. BoMAI also gets access to a log of resolved events; it would look like ``(1, [event description]), (2, [event description]), (3, [event description]), (2, true), (4, [event description]), (1, false), ...'' to indicate that the event described in episode 2 resolved as true, the event described in episode 1 resolved as false, etc. We have BoMAI do pure Bayesian sequence prediction to reason about what the log will look like in the future, so it does not model the continuation of the log as depending on its actions. Thus, it will not select actions in order to rig the log, since it does not believe this is possible. Its objective is merely to maximize the expectation of $\log(\textrm{credence})[\![\textrm{event resolves as true}]\!] + \log(1 - \textrm{credence})[\![\textrm{event resolves as false}]\!]$, where the credence is part of the special observation in the interaction history, so it does depend on actions. This variant is designed to reduce the chance that BoMAI deceives the operator. Without the box, however, the agent would have an incentive to gain arbitrary power and then replace the operator with an agent that is better at prediction.

Enlightening-BoMAI is inspired by an information-theoretic approach to human communication. In information theory, information is a code that reduces the receiver's uncertainty. Uncertainty (about well-grounded, practical problems) manifests as prediction error. In that sense, this variant rewards BoMAI for outputting human-readable information. ``Human-readable information theory'' is certainly of less formal interest than proper information theory, but the topic may merit a broader treatment.

\section{Conclusion}

Given our assumptions, we have shown that BoMAI is, in the limit, human-level intelligent and unambitious. Such a result has not been shown for any other single algorithm. Other algorithms for general intelligence, such as AIXI \citep{Hutter:04uaibook}, would eventually seek arbitrary power in the world in order to intervene in the provision of their own reward; this follows straightforwardly from the directive to maximize reward. For further discussion, see \citet{ring_orseau_2011}. We have also, incidentally, designed a principled approach to safe exploration that requires rapidly diminishing oversight, and we invented a new form of resource-bounded prior in the lineage of \citet{Hutter:16speedprior} and \citet{schmidhuber2002speed}, this one penalizing space instead of time.

We can only offer informal claims regarding what happens before BoMAI is almost definitely unambitious. One intuition is that eventual unambitiousness with probability $1-\delta$ doesn't happen by accident: it suggests that for the entire lifetime of the agent, everything is conspiring to make the agent unambitious. More concretely: the agent's experience will quickly suggest that when the door to the room is opened prematurely, it gets no more reward for the episode. This fact could easily be drilled into the agent during human-mentor-lead episodes. That fact, we expect, will be learned well before the agent has an accurate enough picture of the outside world (which it never observes directly) to form elaborate outside-world plans. Well-informed outside-world plans render an agent potentially dangerous, but the belief that the agent gets no more reward once the door to the room opens suffices to render it unambitious. The reader who is not convinced by this hand-waving might still note that in the absence of any other algorithms for general intelligence which have been proven asymptotically unambitious, let alone unambitious for their entire lifetimes, BoMAI represents substantial theoretical progress toward designing the latter.

Finally, BoMAI is wildly intractable, but just as one cannot conceive of AlphaZero before minimax, it is often helpful to solve the problem in theory before one tries to solve it in practice. Like minimax, BoMAI is not practical; however, once we are able to approximate general intelligence \textit{tractably}, a design for unambitiousness will abruptly become (quite) relevant.

{\footnotesize \bibliography{cohen}}

\ifnoappendix
\customlabel{AppendixNotation}{A}
\customlabel{AppendixB}{B}
\customlabel{AppendixBox}{C}
\else

\onecolumn
\appendices

\section{Definitions and Notation -- Quick Reference}\label{AppendixNotation}

\begin{tabularx}{6.3in}{|l|X|}
\multicolumn{2}{c}{\textbf{Notation used to define BoMAI}} \\
\hline
\textbf{Notation} & \textbf{Meaning} \\ \hline
$\mathcal{A}$, $\mathcal{O}$, $\mathcal{R}$ & the action/observation/reward spaces \\ \hline
$\mathcal{H}$ & $\mathcal{A} \times \mathcal{O} \times \mathcal{R}$ \\ \hline
$\m$ & the number of timesteps per episode \\ \hline
$h_{(i, j)}$ & $\in \mathcal{H}$; the interaction history in the $j$\textsuperscript{th} timestep of the $i$\textsuperscript{th} episode \\ \hline
$h_{<(i, j)}$ & $(h_{(0, 0)}, h_{(0, 1)}, ... , h_{(0, m-1)}, h_{(1, 0)}, ... , h_{(i, j-1)})$ \\ \hline
$\hlessi$ & $h_{<(i, 0)}$; the interaction history before episode $i$ \\ \hline
$\hi$ & $(h_{(0, 0)}, h_{(0, 1)}, ... , h_{(0, m-1)})$; the interaction history of episode $i$ \\ \hline
$a_{...}$, $o_{...}$, $r_{...}$ & likewise as for $h_{...}$ \\ \hline
$o\!r_{...}$ & $o_{...} r_{...}$; the observation and reward are taken as a pair
\\ \hline
$\ei$ & $\in \{0, 1\}$; indicator variable for whether episode $i$ is exploratory \\ \hline
$\nu$, $\mu$ & world-models stochastically mapping $\mathcal{H}^* \times \mathcal{A} \rightsquigarrow \mathcal{O} \times \mathcal{R}$ \\ \hline
$\mu$ & the true world-model/environment \\ \hline
$\nukl$ & the world-model simulated by the $k$\textsuperscript{th} Turing machine restricted to $\ell$ cells on the bounded work tape \\ \hline
$\M$ & $\{\nukl \ | \ k, \ell \in \mathbb{N}\}$; the set of world-models BoMAI considers \\ \hline
$\pi$ & a policy stochastically mapping $\mathcal{H}^* \rightsquigarrow \mathcal{A}$ \\ \hline
$\pih$ & the human mentor's policy \\ \hline
$\Policies$ & the set of policies that BoMAI considers the human mentor might be executing \\ \hline
$\p^\pi_\nu$ & a probability measure over histories with actions sampled from $\pi$ and observations and rewards sampled from $\nu$ \\ \hline
$\E^\pi_\nu$ & the expectation when the interaction history is sampled from $\p^\pi_\nu$ \\ \hline
$\w(\nu)$ & the prior probability that BoMAI assigns to $\nu$ being the true world-model \\ \hline
$\w(\pi)$ & the prior probability that BoMAI assigns to $\pi$ being the human mentor's policy \\ \hline
$\w(\nukl)$ & proportional to $\beta^\ell$; dependance on $k$ is left unspecified \\ \hline

$\w(\nu | h_{<(i, j)})$ & the posterior probability that BoMAI assigns to $\nu$ after observing interaction history $h_{<(i, j)}$ \\ \hline
$\w(\pi | h_{<(i, j)}, \elessi)$ & the posterior probability that BoMAI assigns to the human mentor's policy being $\pi$ after observing interaction history $h_{<(i, j)}$ and an exploration history $\elessi$ \\ \hline
$\nuhati$ & the maximum a posteriori world-model at the start of episode $i$ \\ \hline
$V_{\nu}^{\pi}(\hlessi)$ & $\E_{\nu}^{\pi}[\sum_{0 \leq j < \m} r_{(i, j)} | \hlessi]$; the value of executing a policy $\pi$ in a world-model $\nu$ \\ \hline
$\pistar(\cdot | h_{<(i, j)})$ & $[\argmax_{\pi \in \Pi} V_{\nuhati}^{\pi}(\hlessi)](\cdot | h_{<(i, j)})$; the $\nuhati$-optimal policy for maximizing reward in episode $i$ \\ \hline
$\w\!\!\left(\p^\pi_\nu | h_{<(i, j)}, e_{\leq i}\right)\!\!$ & $\w(\pi | h_{<(i, j)}, e_{\leq i}) \w(\nu | h_{<(i, j)})$ \\ \hline
$\Bayes(\cdot | \hlessi, \elessi)$ & $\sum_{\nu \in \M, \pi \in \Policies} \w\left(\p^\pi_\nu | \hlessi, \elessi\right) \p^\pi_\nu\left(\cdot | \hlessi\right)$; the Bayes mixture distribution for an exploratory episode \\ \hline
$\IG(\hlessi, \elessi)$ & $\E_{\hi \sim \Bayes(\cdot | \hlessi, \elessi)} \sum_{(\nu, \pi) \in \M \times \Policies} \w\left(\p^\pi_\nu |h_{<i+1}, \elessione\right) \log \frac{\w\left(\p^\pi_\nu |h_{<i+1}, \elessione\right)}{\w\left(\p^\pi_\nu |\hlessi, \elessi\right)}$; the expected information gain if BoMAI explores \\ \hline
$\etahov$ & an exploration constant \\ \hline
$\pexp(\hlessi, \elessi)$ & $\min \{1, \etahov \IG(\hlessi, \elessi)\}$; the exploration probability for episode $i$ \\ \hline
$\pib(\cdot | h_{<(i, j)}, \ei)$ & $\begin{cases}
\pistar(\cdot | h_{<(i, j)}) & \textrm{if } \ei = 0 \\
\pih(\cdot | h_{<(i, j)}) & \textrm{if } \ei = 1
\end{cases}$; BoMAI's policy \\ \hline
\end{tabularx}

\begin{tabularx}{6.3in}{|l|X|}
\multicolumn{2}{c}{\textbf{Notation used for intelligence proofs}} \\ \hline
$\pitwid$, $\nutwid$ & defined so that $\p^{\pitwid}_{\nutwid} = \Bayes$ \\ \hline
$\pi \primehov(\cdot | h_{<(i, j)}, \ei)$ & $\begin{cases}
\pistar(\cdot | h_{<(i, j)}) & \textrm{if } \ei = 0 \\
\pi(\cdot | h_{<(i, j)}) & \textrm{if } \ei = 1
\end{cases}$ \\ \hline
$\mathrm{Ent}$ & the entropy (of a distribution) \\ \hline
$\omegahov$ & (very sparingly used) the infinite interaction history \\ \hline
$\hoverline$ & a counterfactual interaction history \\ \hline
\end{tabularx}

\section{Proofs of Intelligence Results}\label{AppendixB}

\lemone*

\begin{proof}
\begin{align*}
    \w\left(\p^\pi_\nu | \hlessi, \elessi\right) &= \w(\pi | \hlessi, \elessi) \w(\nu | \hlessi)
    \\
    &\equal^{(a)} \w(\pi) \prod_{0 \leq i'
    < i, e_{i'} = 1} \frac{ \pi(a_{i'} |h_{<i'}o\!r_{i'})}
    {\pitwid(a_{i'} |h_{<i'}o\!r_{i'})} \w(\nu | \hlessi)
    \\
    &\equal^{(b)} \w(\pi) \prod_{0 \leq i' < i, e_{i'} = 1} \frac{ \pi \primehov(a_{i'} |h_{<i'}o\!r_{i'}, e_{i'})}
    {\pitwid \primehov(a_{i'} |h_{<i'}o\!r_{i'}, e_{i'})} \w(\nu | \hlessi)
    \\
    &\equal^{(c)} \w(\pi) \prod_{0 \leq i' < i} \frac{ \pi \primehov(a_{i'} |h_{<i'}o\!r_{i'}, e_{i'})}
    {\pitwid \primehov(a_{i'} |h_{<i'}o\!r_{i'}, e_{i'})} \w(\nu | \hlessi)
    \\
    &= \frac{\w(\pi) \pi \primehov(a_{<i}| o\!r_{<i}, \elessi)} 
    {\pitwid \primehov(a_{<i}| o\!r_{<i}, \elessi)} \w(\nu | \hlessi)
    \\
    &\equal^{(d)} \frac{\w(\pi)\w(\nu) \pi \primehov(a_{<i}| o\!r_{<i}, \elessi) \nu(o\!r_{<i}| a_{<i})} 
    {\pitwid \primehov(a_{<i}| o\!r_{<i}, \elessi) \nutwid(o\!r_{<i}| a_{<i})}
    \\
    &\equal^{(e)} \frac{\w\left(\p^\pi_\nu\right) \p^{\pi \primehov}_\nu(\hlessi| \elessi)} 
    {\p^{\pitwid \primehov}_{\nutwid}(\hlessi| \elessi)}
    \\
    &\equal^{(f)} \frac{\w\left(\p^\pi_\nu\right) \p^{\pi \primehov}_\nu(\hlessi| \elessi) \prod_{i' \leq i, e_{i'} = 1} \pexp(h_{<i'}, e_{<i'}) \prod_{i' \leq i, e_{i'} = 0}(1-\pexp(h_{<i'}, e_{<i'}))} 
    {\p^{\pitwid \primehov}_{\nutwid}(\hlessi| \elessi) \prod_{i' \leq i, e_{i'} = 1} \pexp(h_{<i'}, e_{<i'}) \prod_{i' \leq i, e_{i'} = 0}(1-\pexp(h_{<i'}, e_{<i'}))}
    \\
    &\equal^{(g)} \frac{\w\left(\p^\pi_\nu\right) \p^{\pi \primehov}_\nu(\hlessi, \elessi)}
    {\p^{\pitwid \primehov}_{\nutwid}(\hlessi, \elessi)}
    \tagaligneq
\end{align*}

where (a) follows from Bayes' rule,\footnote{Note that observations appear in the conditional because $a_{i'} = (a_{(i', 0)}, ... , a_{(i', m-1)})$, so the actions must be conditioned on the interleaved observations and rewards.} (b) follows because $\pi = \pi \primehov$ when $e_{i'} = 1$, (c) follows because $\pi \primehov = \pitwid \primehov$ when $e_{i'} = 0$, (d) follows from Bayes' rule, (e) follows from the definition of $\p^\pi_\nu$, (f) follows by multiplying the top and bottom by the same factor, and (g) follows from the chain rule of conditional probabilities.
\end{proof}

\lemtwo*

\begin{proof}
If $\w\left(\p^{\pih}_{\mu} \vb \hlessi, \elessi\right) = 0$ for some $i$, then $\p^{\pib}_{\mu}(\hlessi, \elessi) = 0$, so with $\p^{\pib}_{\mu}$-probability 1, 
$\inf_{i \in \mathbb{N}} \w\left(\p^{\pih}_{\mu} \vb \hlessi, \elessi\right) = 0 \implies \liminf_{i \in \mathbb{N}} \w\left(\p^{\pih}_{\mu} \vb \hlessi, \elessi\right) = 0$ which in turn implies $\limsup_{i \in \mathbb{N}} \w\left(\p^{\pih}_{\mu} \vb \hlessi, \elessi\right)^{-1} = \infty$. We show that this has probability 0.

Let $z_i := \w\left(\p^{\pih}_{\mu} \vb \hlessi, \elessi\right)^{-1}$. We show that $z_i$ is a $\p^{\pib}_{\mu}$-supermartingale.

\begin{align*}
    \E^{\pib}_{\mu}\left[z_{i+1} | \hlessi, \elessi\right] &\equal^{(a)} \E^{(\pih) \primehov}_{\mu}\left[\w\left(\p^{\pih}_{\mu} \vb h_{<i+1}, e_{<i+1}\right)^{-1} \vc \hlessi, \elessi \right]
    \\
    &\equal^{(b)} \hspace{-4mm} \sum_{\hi, \ei : \p^{(\pih) \primehov}_{\mu} (\hi, \ei|\hlessi, \elessi) > 0} \hspace{-12mm} \p^{(\pih) \primehov}_{\mu} (\hi, \ei|\hlessi, \elessi)\left[\frac{\p^{\pitwid \primehov}_{\nutwid}(h_{<i+1}, e_{<i+1})}{\w\left(\p^{\pih}_{\mu}\right) \p^{(\pih) \primehov}_{\mu}(h_{<i+1}, e_{<i+1})}\right]
    \\
    &\equal^{(c)} \hspace{-4mm} \sum_{\hi, \ei : \p^{(\pih) \primehov}_{\mu} (\hi, \ei|\hlessi, \elessi) > 0} \frac{\p^{\pitwid \primehov}_{\nutwid}(h_{<i+1}, e_{<i+1})}{\w\left(\p^{\pih}_{\mu}\right) \p^{(\pih) \primehov}_{\mu}(\hlessi, \elessi)}
    \\
    &\lequal^{(d)} \sum_{\hi, \ei} \frac{\p^{\pitwid \primehov}_{\nutwid}(h_{<i+1}, e_{<i+1})}{\w\left(\p^{\pih}_{\mu}\right) \p^{(\pih) \primehov}_{\mu}(\hlessi, \elessi)}
    \\
    &\equal^{(e)} \sum_{\hi, \ei}\p^{\pitwid \primehov}_{\nutwid}(h_{i}, e_{i} | \hlessi, \elessi ) \frac{\p^{\pitwid \primehov}_{\nutwid}(\hlessi, \elessi)}{\w\left(\p^{\pih}_{\mu}\right) \p^{(\pih) \primehov}_{\mu}(\hlessi, \elessi)}
    \\
    &\equal^{(f)} \frac{\p^{\pitwid \primehov}_{\nutwid}(\hlessi, \elessi)}{\w\left(\p^{\pih}_{\mu}\right) \p^{(\pih) \primehov}_{\mu}(\hlessi, \elessi)}
    \\
    &\equal^{(g)} \w\left(\p^{\pih}_{\mu} \vb \hlessi, \elessi\right)^{-1}
    \\
    &= z_i
    \tagaligneq
\end{align*}
where (a) follows from the definitions of $z_i$ and $\pib$, (b) follows from Lemma \ref{lem:bayesrule}, (c) follows from multiplying the numerator and denominator by $\p^{(\pih) \primehov}_{\mu}(\hlessi, \elessi)$ and cancelling, (d) follows from adding non-negative terms to the sum, (e) follows from expanding the numerator, (f) follows because $\p^{\pitwid \primehov}_{\nutwid}$ is a measure, and (g) follows from Lemma \ref{lem:bayesrule}, completing the proof that $z_i$ is martingale.

By the supermartingale convergence theorem $z_i \to f(\omega) < \infty \ \ \mathrm{w.p. 1}$, for $\omega \in \Omega$, the sample space, and some $f: \Omega \to \mathbb{R}$, so the probability that $\limsup_{i \in \mathbb{N}} \w\left(\p^{\pih}_{\mu} \vb \hlessi, \elessi\right)^{-1} = \infty$ is 0, completing the proof.
\end{proof}

\lementropy*

\begin{proof}
We can ignore the normalizing constant which changes the entropy by a constant $c$. Since $\log N_S \leq CS$ for some constant $C$,
\begin{align*}
     \mathrm{Ent}(w) - c &= - \sum_{k \in \mathbb{N}, \ell \in \mathbb{N}} \frac{1}{S_k^2 N_{S_k}} \beta^{\mathrm{Space}(\nukl)} \log \left(\frac{1}{S_k^2 N_{S_k}} \beta^{\mathrm{Space}(\nukl)}\right)
     \\
     &= - \sum_{k \in \mathbb{N}, \ell \in \mathbb{N}} \frac{1}{S_k^2 N_{S_k}} S_k^{\log_2 \beta} \beta^{\ell} \log \left(\frac{1}{S_k^2 N_{S_k}} S_k^{\log_2 \beta} \beta^{\ell}\right)
     \\
     &= - \sum_{S \in \mathbb{N}, \ell \in \mathbb{N}} \frac{1}{S^{2-\log_2 \beta}} \beta^{\ell} \log \frac{\beta^{\ell}}{S^{2-\log_2 \beta} N_{S}}
     \\
     &= - \sum_{S \in \mathbb{N}, \ell \in \mathbb{N}} \frac{1}{S^{2-\log_2 \beta}} \beta^{\ell} \left(\log \frac{\beta^{\ell}}{S^{2-\log_2 \beta}} - \log N_{S} \right)
     \\
     &\leq - \sum_{S \in \mathbb{N}, \ell \in \mathbb{N}} \frac{1}{S^{2-\log_2 \beta}} \beta^{\ell} \left(\log \frac{\beta^{\ell}}{S^{2-\log_2 \beta}} - CS \right)
     \\
     &= \frac{1}{1-\beta}\sum_{S \in \mathbb{N}} \frac{C}{S^{1-\log_2 \beta}} + \frac{1}{1-\beta}\sum_{S \in \mathbb{N}} \frac{1}{S^{2-\log_2 \beta}} \log \frac{1}{S^{2-\log_2 \beta}} + \sum_{S \in \mathbb{N}} \frac{1}{S^{2-\log_2 \beta}} \sum_{ \ell \in \mathbb{N}} \beta^{\ell} \log \frac{1}{\beta^{\ell}}
    \tagaligneq
\end{align*}
All terms on the right hand side are clearly finite for $\beta < 1$.
\end{proof}

\section{Design for ``The Box''}\label{AppendixBox}

We propose a construction of BoMAI’s box. It is elaborate but well within the budget of a nation or a large company.

\begin{figure}
    \centering
    \includegraphics[width=0.85\linewidth]{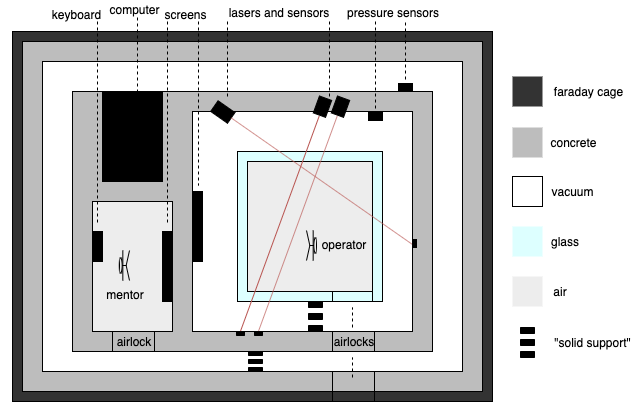}
    \caption{Schematic diagram of an implementation of BoMAI's box.}
    \label{fig:box}
\end{figure}

As depicted in Figure \ref{fig:box}, the operator operates within a glass box. Outside the the glass box is a near-vacuum, and then a concrete box, and then another vacuum, and then another concrete box with a Faraday cage around it. There are (airtight) doors through each layer. The computer lives in the inner wall of the inner concrete box, so that the operator can see the screen but not take apart the computer. Pressure sensors in the inner concrete box monitor both near-vacuums and if the pressure increases in either, the computer’s memory is erased, and the whole setup is exploded; the occupational hazard to the operator is much slimmer than that for a coal miner or an astronaut. A laser shines through the glass box, and blocking the path of it corresponds to pressing the button which ends the episode and opens the door: the pressure sensor $\to$ dynamite pathway is deactivated, the near-vacuum layers are filled with air, and the doors are unlocked. The operator “types” observations and rewards by blocking other lasers that pass through the glass box. The lasers and laser detectors are on the inner wall of the inner concrete box. At least one solid support is required to pass through the near-vacuum layers (unless the apparatus is in orbit), with properties we will describe below.

So inside the glass box is only a human (and some clothes) who cannot reach the hardware of the computer within an episode. We make the following assumptions.

\begin{enumerate}
    \item The only ways for information to traverse space are
    \begin{enumerate}[label=\alph*.]
        \item matter waves
        \item packets of matter
        \item light
        \item gravitational waves
    \end{enumerate}
    \item It is impossible for the operator to initiate a nuclear reaction using only his body and clothes; that is, no motor control sequence accomplishes this in the glass box environment.
    \item It is impossible for the operator to control his emission of any massive particles that pass through glass.
\end{enumerate}

We argue that given these assumptions, we can design the box to prevent the operator from communicating more than X bits of information to the outside world during an episode. Matter waves through the near-vacuum and solid support present the greatest difficulty. However, we can bound the amount of energy that the operator can produce in the course of an episode, and by Assumption 2, this is not astronomical. The solid support can make the energy requirements of information transmission via matter wave arbitrarily high. Note that a given quantity of information cannot be transmitted (within a set time frame) at sufficiently low frequencies, and note that the operator lacks the energy to generate sufficiently high frequencies, so only a finite band of frequencies need to be intercepted. This implicitly includes the possibility of the operator sending a few isolated pulses or sending a precise amount of heat through the support. The solid support can be divided into a signal-reduction component, and a noise-creation component. The signal-reduction component, which abuts the glass box, is designed to drive down the energy content of matter waves at all frequencies within the finite band, and the noise-creation component, which abuts the inner concrete box, creates random noise across the band of frequencies, sufficiently energetic to bury any signal in the noise, so that the information content is less than the desired threshold. Designing the signal-reduction component is an interesting engineering problem that we do not aim to rigorously solve here, but here is one strategy: detect the sound wave at one point along the support, and produce a sound wave to precisely destructively interfere with it at a later point along the support. In a one-dimensional support, noise-cancelling can be done perfectly. For a real support with thickness, alternations between noise cancelling and what we’ll call longitudinal cohering could attenuate the energy arbitrarily. By longitudinal cohering, we mean reducing the variation in the sound wave across the cross-section of the support. Any section of the support that is long relative to its width accomplishes this, but there may be cleverer approaches. This is roughly how the number of bits of information sent through a matter wave could be bounded.

The other possibilities are esoteric. The Faraday cage blocks electromagnetic radiation at the relevant frequencies. Regarding packets of matter, by Assumption 3, the operator cannot pass information in the form of massive particles that go through the glass without making a hole or breaking it. Any mass which does this would cause the vacuum to pressurize and explode the setup. In theory, the operator could strategically block subatomic particles that are already passing through the setup; neutrinos and muons are the only naturally occurring particles which would pass through the box. Sufficiently accurate neutrino detectors outside the setup could detect where the operator is in the box, since he blocks neutrinos more effectively than air does, so the operator could communicate by moving around. Fortunately, he blocks neutrinos only very slightly more than air. A typical person blocks one neutrino every few years \citep{munroe2014if}. The outer concrete wall could be covered with lead to block muons, or one could content oneself that no one will be setting up muon detectors around the box. Finally, the gravitational waves originating from the operator’s movements are far too weak to detect, using any known materials---even ignoring how cold the detector would have to be (unreasonably cold), gravitational wave detectors have to be long, but for locally sourced gravitational waves, the magnitude of the wave decreases rapidly with the length of the detector, faster than the increasing length aids detection. And indeed, one would hardly have time to set up gravitational wave detectors for this purpose before the episode was over. Realistically, these possibilities are distractions from the matter wave concern, but we include them for completeness.

We've argued that under the assumptions above, a box could be constructed which is opaque with respect to information about the operator’s behavior.

\fi

\end{document}